\documentclass[A4paper,12pt]{article}
\usepackage{amsthm}
\usepackage{mathtools}
\DeclarePairedDelimiter{\ceil}{\lceil}{\rceil}
\usepackage{multirow}
\usepackage{graphicx} % more modern
\usepackage{subfigure}
\usepackage{amsmath,amssymb}
\usepackage{array}
\usepackage{rotating}
\newcolumntype{L}[1]{>{\raggedright\let\newline\\\arraybackslash\hspace{0pt}}m{#1}}
\newcolumntype{C}[1]{>{\centering\let\newline\\\arraybackslash\hspace{0pt}}m{#1}}
\newcolumntype{R}[1]{>{\raggedleft\let\newline\\\arraybackslash\hspace{0pt}}m{#1}}

\newtheorem{prop}{Proposition}
 \newtheorem{assumption}{Assumption}
\newcolumntype{M}{>{\centering\arraybackslash}m{.44\linewidth}}

\newcommand{\KG}{\text{KG}}
\newcommand{\IE}{\text{IE}}

\newcommand{\UCBE}{\text{UCB-E}}

\newcommand{\Kriging}{\text{Kriging}}

\newtheorem{theorem}{\sc Theorem}[section]
\newtheorem{lemma}[theorem]{\sc Lemma}

\newtheorem{proposition}[theorem]{\sc Proposition}

\newtheorem{definition}[theorem]{\sc Definition}

\begin{document}

\title{Finite-time Analysis for the Knowledge-Gradient Policy}
\author{Yingfei Wang \thanks{
       Department of Computer Science, Princeton University, Princeton, NJ  08540, USA, yingfei@cs.princeton.edu }
       \and
	Warren Powell \thanks{
       Department of Operations Research and Financial Engineering, Princeton University, Princeton, NJ  08540, USA, powell@princeton.edu}}

\maketitle

% REQUIRED
\begin{abstract}

We consider sequential decision problems in which we adaptively choose one of finitely many alternatives and observe a stochastic reward. We offer a new perspective of interpreting Bayesian ranking and selection problems as adaptive stochastic multi-set maximization problems and derive the first finite-time bound of  the knowledge-gradient policy for  adaptive submodular objective functions. In addition, we introduce the concept of prior-optimality  and provide another insight into the performance of the knowledge gradient policy based on the submodular assumption on the value of information.  We demonstrate submodularity for the two-alternative case and provide other conditions for more general problems, bringing out the issue and importance of submodularity in learning problems.   Empirical experiments are conducted to further illustrate the finite time behavior of the knowledge gradient policy. 
\end{abstract}

\section{Introduction}
We consider sequential decision problems in which at each time step, we choose one of finitely many alternatives and observe a random reward. The rewards are independent of each other and follow some unknown probability distribution. One goal can be to identify the alternative with the best expected performance within a limited measurement budget, which is the objective of Bayesian ranking and selection problems.  Ranking and selection problems  are examples of sequential decision making problems with partial information that address the exploration-exploitation trade-off. Since the learner does not know the true distribution of each alternative, it needs to explore the choices that might give good rewards in the future as well as exploit the alternatives that  appear to be better based on previous observations.

Ranking and selection problems arise in many settings. We
may have to choose a type of material that has the best performance,
the features in a laptop or car that produce the
highest sales, or the molecular combination that produces
the most effective drug. Often, the cost of a measurement
may be substantial. Laboratory or field experiments may take a day or several weeks. For this reason, we assume we
have a limited budget for making measurements.

 Raiffa and Schlaifer established the Bayesian framework for R$\&$S problems \cite{raiffa1961applied}.  Several two-stage and sequential procedures exist for selecting the best alternative.  Branke et. al   made a thorough comparison of several fully sequential sampling procedures \cite{branke2007selecting}. They indicate that the optimal computing budget allocation (OCBA) \cite{chen1996gradient,chen2000simulation, he2007opportunity} and value of information procedures (VIP) \cite{chick2001new} perform quite well and better than a deterministic or two-stage policy \cite{chen2006efficient}. Another single-step Bayesian look-ahead policy first introduced by \cite{gupta1996bayesian} and then further studied by \cite{frazier2008knowledge} is called the ``knowledge-gradient policy" (KG).  It chooses to measure the  alternative that maximizes the single-period expected  value of  information.  Whereas the above mentioned policies assumed an independent normal or one-dimensional Wiener process prior on the alternatives' true means,  Frazier et. al modified the knowledge-gradient policy to handle correlated multivariate normal belief on the mean values of these rewards \cite{frazier2009knowledge}.

A similar field is the multi-armed bandit problem, which  were originally studied under Bayesian assumptions \cite{gittins1979bandit}.  A widely used class of policies for multi-armed bandit problems
is called {\it upper confidence bounding} policies (UCB). Different UCB-type variants have been developed for many types of reward distributions and
have provable logarithmic regret bounds \cite{lai1985asymptotically, agrawal1995sample, auer2002finite,
kleinberg2010regret,bubeck2012bandits}. By contrast, knowledge gradient policies, which enjoy some nice theoretical properties, have never been characterized by the type of regret bounds for which UCB policies are famous.

This paper makes the following contributions:   We first establish the connection between  Bayesian ranking and selection problem and adaptive stochastic multi-set function maximization problems where each multi-set corresponds to a set of selected alternatives. The multi-set representation captures our ability to evaluate the same alternative more than once.  This  new perspective  offers a new line of analysis for the properties of value-of-information policies. We derive the first finite-time bound for the knowledge gradient policy for R\&S problems under the assumption that the utility function is adaptive submodular.  However,  pathwise adaptive submodularity can fail in offline learning settings when the utility function itself involves a maximum. To this end, instead of the  pathwise behavior analyses of the utility function, we further study its average behavior by taking expectations over the observations given any fixed sample allocation, resulting in a well-known quantity: the value of information. As a result,  we introduce the  concept of the prior-value of a policy and analyze the prior-optimality of the KG policy to provide another insight into its performance based on the submodular assumption of the value of information that is weaker than  adaptive submodularity. To accomplish this, we build on the general structure of the analysis of greedy algorithms given in \cite{nemhauser1978analysis} and \cite{golovin2010adaptive}.  We demonstrate submodularity for the two-alternative case and provide other conditions for more general problems, filling in a gap in the analysis of the knowledge gradient policy. Finally, we propose experiments to illustrate our theoretical analysis on the finite time behavior of the knowledge gradient policy. We further compare the KG policy with other policies with or without theoretical guarantees. Aside from the fact that the KG policy  performs competitively with or significantly
better than other policies especially in early iterations,  we draw the conclusion that there is no universal best policy for all problem classes, which means that theoretical guarantees are not by themselves reliable indicators of which policy is best for a particular problem class and empirical experiments are needed to better understand their finite time performance.  

This paper is organized as follows. In section 2, we lay out the mathematical models for Bayesian ranking and selection problems. In section 3, we describe the knowledge gradient policies.
 In section 4, we provide finite-time analyses of the knowledge gradient policy from two directions: the posterior optimality and the prior optimality.  In section 5,  we analyze the submodularity of the two-alternative case  and provide other conditions for more general problems, bringing out the issue and importance of submodularity in leaning problems. Finally, in section 6 we present finite-time performance results and analyses of various policies  for  R\&S problems.

\section{ Ranking and Selection Problems} \label{RG}
Suppose we have a collection $\mathcal{X}$ of $M$ alternatives (where $M$ might be quite large), each of which can be measured sequentially to estimate its  unknown  mean $\mu_x$. We assume normally distributed measurement noise with known variance $\sigma_W^2$.  We first introduce the model for independent normal beliefs. We begin with a normally distributed Bayesian prior belief on the sampling means that is independent across alternatives, $\mu_x \sim \mathcal{N}(\theta_x^0, \sigma_x^0)$.  At the $n$th iteration,  we use some measurement policy $\pi$ to choose one alternative $x^n$ and observe $W^{n+1}_{x^n} \sim \mathcal{N}(\mu_{x^n}, \sigma_W)$.

For convenience, we introduce the $\sigma$-algebras  $\mathcal{F}^{n}$ for any $n =0,1,...,N-1$ which is formed by the previous $n$ measurement choices and outcomes, $x^0, W^1,..., x^{n-1},W^n$.  We define $\theta^n_x=\mathbb{E}[\mu_x|\mathcal{F}^n]$ and $(\sigma^n_x)^2=\text{Var}[\mu_x|\mathcal{F}^n]$. Then conditionally on $\mathcal{F}^n$, $\mu_x \sim \mathcal{N}(\theta^n_x,\sigma^n_x)$. Let $\beta_x^n=\frac{1}{(\sigma_x^n)^2}$ be the conditional precision of $\mu_x$ and  our state of knowledge be  $S^n =(\theta^n_x,
\beta^n_x)_{x\in \mathcal{X}}$. We will use $\mathcal{F}^n$ and $S^n$ interchangeably. After the $n$th measurement we update our beliefs using Bayes' rule:
\begin{displaymath}
\theta_x^{n+1} = \left\{ \begin{array}{ll}
 \frac{\beta_x^n\theta_x^n+\beta^WW^{n+1}}{\beta^n+\beta^W} & \textrm{if $x^n=x$}\\
 \theta_x^{n} & \textrm{otherwise},
  \end{array} \right.
~~~~~~~~
\beta_x^{n+1}=\left\{ \begin{array}{ll}
 \beta^n+\beta^W & \textrm{if $x^n=x$}\\
 \beta_x^{n} & \textrm{otherwise},
  \end{array} \right.
\end{displaymath}
where $\beta^W=1/\sigma_W^2$.

We may impose correlated beliefs between alternatives in order to strengthen the effect of each measurement. Starting from a prior distribution $\mathcal{N}(\theta^0, \Sigma^0)$ and after measurement $W^{n+1}$ of alternative $x$, a posterior
distribution on the beliefs are calculated by:
\begin{align}\label{aabb}
    \theta^{n+1} &= \Sigma^{n+1}\left( \left(\Sigma^n\right)^{-1}\theta^n + \beta^W W^{n+1} e_x\right), \\ \label{bb}
    \Sigma^{n+1} &= \left( \left(\Sigma^n\right)^{-1} + \beta^W e_x e_x^T\right)^{-1},
\end{align}
where $e_x$ is the vector with 1 in the entry corresponding to alternative $x$ and 0 elsewhere. $S^n=(\theta^n, \Sigma^n)$ is then our state of knowledge in this case.

A decision function $X^{\pi}(S^n)$ is defined as a mapping from the knowledge state to $\mathcal{X}$. We refer to the decision function $X^{\pi}$ and the policy $\pi$ interchangeably.

If we are limited to $N$ measurements,  the objective  is to maximize the expected reward of the final recommended alternative:
\begin{equation} \label{offobj}
\max_{\pi \in \Pi}\mathbb{E}\left[\mu_{x^\pi}\right],
\end{equation}
where $x^\pi=\arg \max_{x\in \mathcal{X}} \theta_x^N$ and $x^n=X^{\pi}(S^n)$ for $0\le n<N$.

\section{Knowledge Gradient}
For R$\&$S problems, the knowledge gradient is a policy that at the $n$th iteration chooses its $(n+1)$st measurement from $\mathcal{X}$ to maximize the single-period expected increase in value \cite{frazier2008knowledge,frazier2009knowledge}.  To be more specific, the value of being in state $S^n$ is $ \max_{x\in \mathcal{X}} \theta_x^n$.   If we choose to measure $x^n=x$ right now, allowing us to observe $W_x^{n+1}$,  then we transition to a new state of knowledge $S^{n+1}=(\theta^{n+1}, \Sigma^{n+1})$. At iteration $n$,  $\theta_x^{n+1}$ is a random variable since we do not yet know what $W^{n+1}$ is going to be. We would like to choose $x$ at iteration $n$ which maximizes the expected value of $\max_{x \in \mathcal{X}} \theta_x^{n+1}$.
We can think of this as choosing an alternative to maximize the incremental value, given by
\begin{equation}\label{TKG}
\nu_x^{\KG,n}=\mathbb{E}[\max_{x'}\theta_{x'}^{n+1}-  \max_{x'} \theta^{n}_{x'}| x^n=x, S^n].
\end{equation}
The knowledge gradient policy  $X^{\KG}(S^{n})$ is defined by
\begin{equation} \label{kg1}
X^{\KG}(S^{n})= \arg \max_{x \in \mathcal{X}} \nu_x^{\KG,n}.
\end{equation}

The knowledge gradient policy can handle the presence of a variety of belief models such as (generalized) linear \cite{negoescu2011knowledge,Wang2016KG} or  nonparametric \cite{mes2011hierarchical, barut2013optimal}.
 
 The knowledge gradient policy has some nice properties. For Bayesian ranking and selection problems, the knowledge gradient policy is optimal (by definition) if the measurement budget $N=1$.  The knowledge gradient  is guaranteed to find the best alternative as the measurement budget $N$ tends to infinity. If there are only two choices, the knowledge gradient policy is optimal for any measurement budget. The knowledge gradient policy is the only stationary policy that is both myopically and asymptotically optimal. However, the KG has not enjoyed the finite-time bounds that have been popular in the UCB policies.

\section{Finite-time Analysis of the  Knowledge Gradient Policy}

We follow  the general structure of  the analysis of greedy approximation \cite{nemhauser1978analysis} to develop the first finite-time bound for the knowledge gradient policy for R\&S problems  as follows. In Section \ref{reduction}, by interpreting the Bayesian R\&S problems as the adaptive stochastic multi-set maximization problems, we show that the KG policy inherits precisely the performance guarantees of the greedy algorithm for classic submodular maximization problems if the utility function is adaptive submodular.  We theoretically analyze  the adaptive submodular assumption and point out that it can  fail in the ranking and selection problems. In such cases, instead of the pathwise behavior analyses of the utility function, we study its average behavior by taking expectation over the observations in Section \ref{value}. In Section \ref{bou}, we analyze the {\it {prior-optimality}} which provides another insight into the performance of the KG policy based on the submodularity of a well-understood quantity: value of information.

 It is important to note that both the submodular maximization reduction and the theoretical analyses on the prior-optimality are not limited to the specific setup of Gaussian noise in observations and Gaussian prior structure. The theoretical guarantees  are more generally applicable to any prior and measurement noise model as long as the  adaptive submodular assumption or the submodular value of information assumption holds.

\subsection{The Reduction of R\&S to Adaptive Stochastic Multi-set Maximization}\label{reduction}

We first introduce the adaptive stochastic maximization problem. Let $E$ be a finite set of items. Each item $e \in E$ maps to a random outcome of a measurement  $\Phi(e)$ in a  set $O$ of possible values. We define a realization as a function $\phi: E \mapsto O$ representing the observation of each item in the ground set. Under Bayesian interpretation, we assume that there is a known prior probability distribution $p(\phi) := \mathbb{P}(\Phi = \phi)$ over all possible realizations. The adaptive stochastic optimization problem consists of sequentially picking an item $e \in E$, revealing its outcome $\Phi(e)$ and picking the next item. After each pick, the observations so far can be represented as a partial realization $\psi$. A partial realization $\psi$ is consistent with realization $\phi$,  denoted as $\phi \sim \psi$, if all the items selected in $\psi$ have the same outcomes  as in $\phi$. We use  $\text{dom}(\psi)$ to refer to the items observed in $\psi$.  We use the notation $Z^\pi( \phi)$ to denote the set of items chosen by policy $\pi$ under realization $\phi$.

We wish to maximize some utility function  $f: 2^E \times O^E \mapsto \mathbb{R}$ that depends on which items we pick and which states they are in.  The expected utility of a policy $\pi$ is $f_{\text{avg}}(\pi) := \mathbb{E}\big[f\big(Z^\pi(\Phi),\Phi \big)\big]$ where the expectation is taken over the prior distribution $p(\phi)$. The goal of adaptive stochastic set maximization problem is to find an optimal policy $\pi^*$ that maximizes its  expected utility under a cardinality constraint,
$$ \pi^* \in \arg \max_{\pi}f_{\text{avg}}(\pi), \text{ subject to } |Z^{\pi}(\phi)| \le N,$$
where $N$ is the measurement budget. 

It is not obvious to treat the ranking and selection problem in an adaptive stochastic multi-set maximization way of thinking. To see this,  define the ground set $E=\mathcal{X}$.  The outcomes are real numbers with $O=\mathbb{R}$. Each alternative  $e=x$ can be selected multiple times. After each selection, its random outcome $\Phi(e) =W_x \in O$ is revealed. 

 Since the true values $\mu_x$ are random variables, we can let $\varphi$ be a sample realization of the truth with a (correlated) prior distribution $p(\varphi)=\mathcal{N}(\theta^0, \Sigma^0)$. We use the notation $ \phi  \in \Phi$ to denote an realization of the random observations in our problem. The prior probability distribution over the realizations is determined by $p(\varphi)$ and the noise distribution $\mathcal{N}(0, \sigma_W)$. For example, if in the ranking and selection problems each alternative can only be selected once, $\phi : E \mapsto O$.  For multi-selections, one way of defining the realization is by first making  replicas of each item to construct $E'$ and then selecting each $e' \in E'$ at most once.
 
 Consider any sampling allocation  $z = (z_x)_{x\in \mathcal{X}}$, by which we measure alternative $x$ for $z_x \in \mathbb{N}$ times. We use $Z$ to represent its corresponding multi-set.   We use $Z^\pi( \phi ): \Phi \mapsto (\mathcal{X} \times\mathbb{N}) $  to refer to the alternatives selected by $\pi$ under realization $ \phi $. Let $\theta^n$ be our vector of estimates of the means after $n$ measurements according to allocation $Z$ under realization $ \phi $, where $|Z|=n.$  $\theta^n$  can be obtained according to the updating equation \eqref{aabb} and \eqref{bb}, and does not depend on the order of the allocations. It  can thus  be denoted as $\theta^n(Z,  \phi ):  (\mathcal{X} \times\mathbb{N}) \times \Phi \mapsto \mathbb{R}^M$. The next lemma states the equivalence of $\mathbb{E}[\mu_{x^\pi}]$ and $\mathbb{E}[\max_{x}\theta_x^N]$. Hence,   the utility function $f:(\mathcal{X} \times\mathbb{N}) \times \Phi \mapsto \mathbb{R}$ can be defined as $\max_x\theta^n_x(Z,  \phi )$ and $f_{\text{avg}}(\pi) :=\mathbb{E}\Big[\max_{x}\theta_x^N\big(Z^\pi(\Phi), \Phi \big)\Big]$.  The R\&S objective \eqref{offobj}  can then be  re-written as 
 $$ \pi^* \in \arg \max_{\pi}f_{\text{avg}}(\pi), \text{ subject to } |Z^{\pi}( \phi )| \le N.$$

 \begin{lemma}[\cite{powell2012optimal}]
 Let $\pi$ be a policy, and let $x^\pi = \arg\max_x \theta_x^N$ be the alternative selected by the policy. Then $$\mathbb{E}[\mu_{x^\pi}] = \mathbb{E}[\max_{x}\theta_x^N].$$
 \end{lemma}

The definition of the knowledge gradient $\nu_x^{KG,n}$ coincides with the {\it {Conditional Expected Marginal Benefit}} $\Delta(e|\psi)$ defined by \cite{golovin2010adaptive}: $$\Delta(e|\psi) := \mathbb{E}\Big[ f\big(\text{dom}(\psi) \cup \{e\}, \Phi \big) - f\big(\text{dom}(\psi),\Phi\big)|\Phi \sim \psi \Big].$$ The knowledge gradient policy is thus in fact the adaptive greedy policy with uniform item costs, with a slight difference in the ability of selecting each item more than once. We generalize the definition of adaptive monotonicity and adaptive submodularity for set functions given  by  \cite{golovin2010adaptive} to multi-set functions as follows.

\begin{definition}[Adaptive Monotonicity] A function $f:(\mathcal{X} \times\mathbb{N}) \times \Phi \mapsto \mathbb{R}$ is adaptive monotone with respect to distribution $p( \phi )$ if the conditional expected marginal benefit of any item is nonnegative: for all $\psi$ and all $x \in \mathcal{X}$.
 \begin{equation*}
 \Delta(x|\psi) \ge 0.
 \end{equation*}
\end{definition}

\begin{definition}[Adaptive Submodularity] A function $f:(\mathcal{X} \times\mathbb{N}) \times \Phi \mapsto \mathbb{R}$ is adaptive submodular with respect to distribution $p( \phi )$ if for all $\psi$ and $\psi'$ such that $\text{dom}(\psi) \subseteq \text{dom}(\psi')$ and both $\psi, \psi'$ are consistent with some realization $ \phi $ (i.e. $\psi \subseteq \psi'$), we have the conditional expected marginal benefit of any fixed item $x \in \mathcal{X}$  does not increase as more items are selected and observed,
\begin{equation*}
\Delta(x|\psi) \ge \Delta(x|\psi').
\end{equation*}
\end{definition}

Let $\pi^*$ be the optimal policy to R\&S problems. If $f:= \max_x \theta^n_x(Z,  \phi )$ is adaptive monotone and adaptive submodular with respect to the prior distribution $p( \phi )$, then 
$$f_\text{avg}(\text{KG}) > (1- e^{-1})f_\text{avg}(\pi^*).$$

We next show that the instances generated by ranking and selection problems are adaptive monotone.
\begin{lemma} In ranking and selection problems,  the utility function $\max_x\theta_x$ is adaptive monotone with any Gaussian prior.
\end{lemma}
\begin{proof} For any $\psi$, let $n=|\psi|$. Then for any item $x \in \mathcal{X}$,  $\Delta(x|\psi)$ can be rewritten as $\mathbb{E}[\max_{x'}\theta^{n+1}_{x'} - \max_{x'}\theta^{n}_{x'} | x^{n}=x, \mathcal{F}^n]=\nu_x^{\text{KG},n}$. Since for any $x$, $\theta^{n+1}_x=\theta^{n}_x+\tilde{\sigma}(\Sigma^n, x^n)Z^{n+1}$, where $\tilde{\sigma}(\Sigma,x)=\frac{\Sigma e_x}{\sqrt{1/\beta^W+\Sigma_{xx}}}$ and the random variable $Z^{n+1}$ is standard normal when conditioned on $\mathcal{F}^n$ \cite{frazier2009knowledge}. Hence we have $\mathbb{E}[\theta^{n+1}_{x'}|x^{n}=x, \mathcal{F}^n] =\theta^{n}_{x'}$ for any $x'$. By Jensen's inequality, we have  $\Delta(x|\psi)=\nu_x^{\text{KG},n} \ge 0.$
\end{proof}
Even though intuition suggests that the utility function should be adaptive submodular in the amount of information collected,  as we collect more information it is natural to expect that the marginal value
of this information should decrease, yet it is not always the case as shown in the next lemma.  The proof can be found in Appendix \ref{A}.

\begin{lemma} \label{conL}
For any independent normal prior distribution $p(\varphi)$ and nondegenerated noise distribution (i.e. $\sigma^W \neq 0$), there exists $\psi$, $\psi'$ and $x \in \mathcal{X}$ such that  $\psi \subseteq\psi'$ and $\Delta(x|\psi) < \Delta(x|\psi').$ 
\end{lemma}

It can be seen that the adaptive submodular assumption  can fail in the ranking and selection problems with the special utility function $f=\max_{x} \theta_x^n(Z, \phi )$ that involves  maximization  itself. Hence, instead of the above pathwise behavior analyses of the utility function, we would like to study its average behavior by taking the expectation over the observations given any fixed sample allocation $Z$ in the next section.

\subsection {The Value of Information} \label{value}
We define the pathwise value of information $\hat{v}(Z, \phi )$ as the incremental improvement over the best expected value that can be obtained without measurement, which is $\max_{x  \in \mathcal{X}}\theta^0_{x}$, $$\hat{v}(Z, \phi ) := \max_{x \in \mathcal{X}}\theta_x^n(Z, \phi )-\max_{x  \in \mathcal{X}}\theta^0_{x}.$$    The value of information $v(Z)$ is then defined to be 
\begin{eqnarray*}
v(Z):= \mathbb{E}_\Phi[\hat{v}(Z, \Phi)], %=\mathbb{E}[\max_{x}\theta^n_{x}- \max_{x}\theta^0_{x}| Z],
\end{eqnarray*}
where the expectation is taken over the prior distribution  $p(\phi)$. 

The value of information has a long history spanning
the literatures of several disciplines. Stigler considers the value of information in economics when buyers search for the best price  \cite{stigler1961economics}.   Howard   laid the groundwork for the value of information in a decision-theoretic context and spawned a great deal of work in this area \cite{howard1966information}. Yokota and Thompson gives a first comprehensive review of value
of information analyses related to health risk management \cite{yokota2004value}.   Raiffa and Schlaifer  poses the Bayesian R\&S problem and defines the associated value of information  \cite{raiffa1961applied}, which marked the
beginning of a number of literature on
the value of information within Bayesian R\&S and the budgeted learning problem \cite{guttman1964bayesian,kapoor2005learning, chen1996gradient, chick2001new, frazier2008knowledge}.

Since the value of information is a multi-set function,  we first generalize the definitions and properties of submodular set functions described by \cite{nemhauser1978analysis} to submodular multi-set functions.
\begin{definition}
Given a finite set $E$, a real-valued function $g$ on the set of multi-sets over $E$ is called submodular if for all multi-sets $S$ and $T$ whose elements belong to $E$,
$$\rho_x{(S)} \geq \rho_x{(T)}, \forall S \subseteq T, \forall x \in E,$$
where $\rho_x(S)\triangleq g(S \cup \{x\})-g(S)$ is the incremental value of adding element $x$ to the multi-set $S$.

\end{definition}

\begin{prop}\label{a1}
Each of the following statements is equivalent and defines a submodular multi-set function ($S$ pathwiseand $T$ are multi-sets on $E$, $x$, $y \in E$):
\begin{enumerate}
\item $\rho_x{(S)} \geq \rho_x{(T)}, \forall S \subseteq T$ and $\forall x $.
\item $\rho_x(S) \geq \rho_x(S \cup \{y\}), \forall S, x, y$.
\item $g(T)\leq g(S)+ \sum_{x \in T-S} \rho_x(S) - \sum_{x \in S-T} \rho_x(S \cup T-\{x\}), \forall S,T$.
\item $g(T) \leq g(S) + \sum_{x\in T-S} \rho_x(S), \forall S \subseteq  T$.
\end{enumerate}
\end{prop}

This proposition follows from a similar  proof of Proposition 2.1 in \cite{nemhauser1978analysis}. For completeness we provide the proof in Appendix \ref{B}.

It is obvious that if $\theta^n_x(Z,  \phi )$ is adaptive monotone or adaptive submodular with respect to $p( \phi )$, then so does $\hat{v}(Z,  \phi )$. It is also easy to show that if $\theta^n_x(Z,  \phi )$ is adaptive monotone or adaptive submodular with respect to $p( \phi )$, then by the law of total expectation, i.e. $\mathbb{E}[\mathbb{E}[U|V]]=\mathbb{E}[U]$ for any random variables $U$ and $V$, the value of information $v(Z)$ is monotone or submodular. We close this section by showing the monotonicity of the multi-set function $v$ and leave the analyses of submodularity in Section \ref{asvi}.
\begin{lemma}\textup{\textbf{(Monotonicity of the value of information)}} 

For any sampling allocation $Z_1$ and $Z_2$, if $Z_1 \subseteq Z_2$, then $v(Z_1) \leq v(Z_2)$.
\end{lemma}
\begin{proof}
We prove the monotonicity of $v$ by showing $v(Z) \leq v(Z\cup \{x^{n+1}\})$ for any allocation $Z$  (with $\sum_{x \in \mathcal{X}} z_x=n$) and any additional measurement $x^{n+1}$. By the tower property,
\begin{eqnarray*}
v(Z\cup \{x^{n+1}\})-v(Z) 
&=&\mathbb{E}_\Phi[\hat{v}(Z\cup \{x^{n+1}\})]-\mathbb{E}[\hat{v}(Z)]\\
&=&\mathbb{E}_\Phi[\max_{x \in \mathcal{X}}\theta^{n+1}_x(Z\cup \{x^{n+1}\})] - \mathbb{E}[\max_{x \in \mathcal{X}}\theta^n_x(Z)] \\
%&=&\mathbb{E}[\mathbb{E}[\max_{x \in \mathcal{X}}\theta^{n+1}_x | Z,x^{n+1}, \phi ^n]] - \mathbb{E}[\mathbb{E}[\max_{x \in \mathcal{X}}\theta^n_x | Z, \phi ^n]]\\
&=&\mathbb{E}_\Phi[\mathbb{E}[\max_x\theta^{n+1}_x(Z\cup \{x^{n+1}\})- \max_x\theta^{n}_x(Z) | \Phi \sim \psi_Z]]\\
&=&\mathbb{E}_\Phi[\nu_x^{\KG,n}],
\end{eqnarray*}
where $\psi_Z$ is the partial realization with $\text{dom}(\psi_Z)=Z$. The lemma follows from the adaptive monotonicity, $\nu_x^{\KG,n} \geq 0$.
\end{proof}

\subsection{Guarantees on the Prior-optimality of the  Knowledge Gradient Policy}\label{bou}
There are two ways to evaluate the value of a policy.  The first, which we call the {\it posterior view}, conditions on the allocation $Z=Z^\pi(\Phi)$ that would have occurred under policy $\pi$ for each sample path $\phi\in\Phi$.  This is the more conventional approach for evaluating policies.  The second, which we call the {\it prior view}, starts by characterizing the value of an arbitrary allocation $Z$ (before we have seen any sample realizations).

More formally, the classical way to estimate the value of a policy is to calculate the incremental improvement over what we could do before we collect any information, is given by
 \begin{eqnarray}\nonumber
f'_{\text{avg}}(\pi)&=&\mathbb{E}[f(Z^{\pi}(\Phi), \Phi)]-\max_x \theta_x^0. 
\end{eqnarray}
We let $\mathbb{P}(\pi \leadsto Z)$ be the probability that policy $\pi$ produces allocation $Z$.   Since with a fixed budget of $N$ measurements, there are only finite choices of possible allocations, using the tower property, we can condition on the allocation $Z^\pi=Z$ which gives us
 \begin{eqnarray}\nonumber
f'_{\text{avg}}(\pi)&=&\sum_{Z\in\mathcal{Z}^N}\mathbb{P}(\pi \leadsto Z)\bigg(\mathbb{E}[ \max_x \theta^n_x(Z^{\pi}(\Phi), \Phi)|Z^{\pi}=Z]-\max_x\theta_x^0\bigg). \label{post}
\end{eqnarray}
We note that in this method for evaluating a policy (which is the standard method), we only consider allocations $Z$ that are actually produced by policy $\pi$ for the outcomes in $\phi$.  This approach makes it much more difficult to understand the relationship between the allocation $Z$ and the value of a policy.

For this reason, we adopt a different method of evaluating a policy which we term the {\it prior view}.  Since this idea is new, we define it formally as follows
\begin{definition}[The prior-value of a policy]
Let $\mathcal{Z}^n$ be the set of all possible allocations with a limited budget $n$. The value of a policy ${\pi}$ with $N$ measurements is defined as
\begin{eqnarray*}
F^{\pi}&=&\sum_{Z\in \mathcal{Z}^N}\mathbb{P}(\pi \leadsto Z) \bigg(\mathbb{E}_{\Phi}[\max_x \theta^n_x(Z,\Phi)]-\max_x\theta_x^0\bigg)\\
&=&\sum_{Z\in \mathcal{Z}^N}\mathbb{P}(\pi \leadsto Z) v(Z).
\end{eqnarray*}

\end{definition}
In this view, we use the prior probability of an outcome $p(\phi)$ instead of the posterior $p(\phi|Z^{\pi}(\phi)=Z)$ which is conditioned on an allocation $Z$.  The value of this approach is that it writes the value of a policy directly as a function of $v(Z)$, making it easier to study the effect of the properties of $v(Z)$ on the value of a policy.  Intuitively, since a policy could generate different allocations $Z$ for different sample realizations, it is natural to define the value of a policy $\pi$ as the weighted sum of the expected value of information based on all possible allocations $Z$ and the weight should be the probability of occurrence of $Z$ based on policy $\pi$. 

We make the following assumption which is weaker than the adaptive submodularity assumption  and will analyze it  further in Section \ref{asvi}.

\begin{assumption}\label{Ass} The value of information $v$ is a submodular multi-set function on the set of alternatives $\mathcal{X}$ with respect to the prior distribution $p( \phi )$.
\end{assumption}
 
Let $\pi^*$ be the  optimal sequential policy under a budget of $N$ measurements in the sense that the prior-value of $\pi^*$ is the largest. We call it {\it{prior-optimality}}. In what follows,  we first bound KG's sub-{\it {prior-optimality}} in Proposition  \ref{4}:
\begin{equation*}\label{subo}
F^{\pi^*}  \le F^{\KG^{[n]} @ \pi^*} \le F^{\KG^{[n-1]}}+ N (F^{\KG^{[n]}}-F^{\KG^{[n-1]}}),~n=1,2,...,N.
 \end{equation*}
Then we
derive the worst-case bound for the KG policy in Theorem \ref{theo}:
\begin{eqnarray*}
&&\frac{F^{KG}}{F^{\pi^*}}
\geq 1-(\frac{N-1}{N}) ^N \geq \frac{e-1}{e} \approx 0.632.
\end{eqnarray*}

Besides the {\it{posterior optimality}} bound obtained from adaptive stochastic multi-set maximization, 
the {\it {prior-optimality}} provides another insight into the performance of the KG policy based on a well-understood quantity: value of information.

\begin{definition}[Policy concatenation]\cite{golovin2010adaptive} A concatenated policy $\pi=\pi_1 @ \pi_2$ is constructed by running $\pi_1$ to completion, and then running policy $\pi_2$  from a fresh start ignoring all the information collected while running $\pi_1$.

\end{definition}

To be more specific, suppose $\pi_i$ has a budget of $n_i$, $i=1,2$, the first phase is to run $\pi_1$ for $n_1$ iterations starting from $S^0$ and we get a sample realization including decisions and their corresponding measurements.
The second phase is to run $\pi_2$ for $n_2$ measurements starting from $S^0$ and we get another sample realization.
Thus the sample realization of the concatenated process is all the  decisions and their corresponding measurements collected in two phases.
Note here, when running the second policy, we ignore all the information collected during running the first one, but when calculating the value of $\pi_1@\pi_2$, $F^{\pi_1 @\pi_2}$, we use all the information  collected in two phases.

\begin{definition}[Policy truncation] \cite{golovin2010adaptive} For a policy $\pi$,  define the j-truncation $\pi^{[j]}$  of $\pi$ as the policy that runs exactly $(j+1)$ steps under $\pi$'s decision rule and $\pi^{\{j\}}$ as the single step policy that randomly chooses an alternative according to the probability distribution of policy $\pi$'s decision for the $(j+1)$-th step. 
\end{definition}

 We now show that the value of $\pi_1$ is no larger than the value of $\pi_1 @ \pi_2$.
\begin{lemma} \label{mon}
 $F^{\pi_1} \le F^{\pi_2@\pi_1}$ for all policies $\pi_1$ and $\pi_2$ under any prior and probability distribution that
describes a measurement.
\end{lemma}

\begin{proof}
We first show that $F^{\pi_1@\pi_2}= F^{\pi_2@\pi_1}$. In a concatenated policy, the two phases are independent since
no information is shared among the two phases. Hence for a given allocation pair $(Z_1,Z_2)$ where $Z_1\in \mathcal{Z}^{n_1}$, $Z_2\in\mathcal{Z}^{n_2}$,
we have 
\begin{eqnarray*}
\mathbb{P}(\pi_1@\pi_2 \leadsto (Z_1,Z_2))&=&\mathbb{P}(\pi_1 \leadsto Z_1 )
\mathbb{P}(\pi_2 \leadsto Z_2)\\
&=&\mathbb{P}(\pi_2 \leadsto Z_2)\mathbb{P}(\pi_1 \leadsto Z_1 )\\
&=&\mathbb{P}(\pi_2@\pi_1 \leadsto (Z_2,Z_1)).
\end{eqnarray*}
$F^{\pi_1@\pi_2}=F^{\pi_2@\pi_1}$ follows immediately from taking the sum over all possible pairs of $(Z_1,Z_2))$ such that $Z_2\cup Z_1=Z$ for any fixed allocation $Z$.

Therefore $F^{\pi_1} \le F^{\pi_1@\pi_2}$ holds if and only if $F^{\pi_1} \le F^{\pi_2@\pi_1}$. We then finish this proof by showing  $F^{\pi_1} \le F^{\pi_1@\pi_2}$.
We write $F^{\pi_1@\pi_2}-F^{\pi_1} $ as a telescoping sequence

\begin{eqnarray*}
&&F^{\pi_1@\pi_2}-F^{\pi_1}\\ &=& 
\sum_{Z\in\mathcal{Z}^{n_1+n_2}}v(Z)\mathbb{P}(\pi_1 @\pi_2  \leadsto Z)
-\sum_{Z_1\in\mathcal{Z}^{n_1}}v(Z_1)\mathbb{P}(\pi_1  \leadsto Z_1)\label{p1}\\
&=&
\sum_{Z\in\mathcal{Z}^{n_1+n_2}}\sum_{Z_1\cup Z_2=Z}
v(Z)\mathbb{P}(\pi_1  \leadsto Z_1)\mathbb{P}(\pi_2  \leadsto Z_2)\\
&-&\sum_{Z_1\in\mathcal{Z}^{n_1}}\sum_{Z_2\in\mathcal{Z}^{n_2}}
v(Z_1)\mathbb{P}(\pi_1  \leadsto Z_1)\mathbb{P}(\pi_2  \leadsto Z_2)\label{p2}\\
&=&\sum_{Z_1\in\mathcal{Z}^{n_1}}\sum_{Z_2\in\mathcal{Z}^{n_2}}\Big{[}
v(Z_1 \cup Z_2)-v(Z_1)\Big{]}\mathbb{P}(\pi_1  \leadsto Z_1)\mathbb{P}(\pi_2  \leadsto Z_2)\\
&\ge& 0,
\end{eqnarray*}

where the second equality holds due to the same reason as in the proof above for $F^{\pi_1@\pi_2}=F^{\pi_2@\pi_1}$ and the third equality is just the same summation in different orders.
The last inequality holds because of the monotonicity of multi-set function $v$. 
\end{proof}

Based on the monotonicity of $v$ and a similar argument as in Proposition \ref{mon}, $F$ is non-decreasing with respect to the number of measurements. Thus the more measurements, the better the policy. Hence $\pi^*$ has exactly $N$ measurements. We have the following sub-optimality bound on $KG$'s prior-value. For a proof see Appendix \ref{C}.
\begin{proposition} \label{4}
Let $\rho^{   \text{KG}  ,n}=F^{   \text{KG}  ^{[n]}}-F^{   \text{KG}  ^{[n-1]}},$ then
\begin{eqnarray} \nonumber
F^{\pi^*} \leq F^{\   \text{KG}  ^{[n-1]}@\pi^*} &\leq&  F^{\   \text{KG}  ^{[n-1]}}+ N\rho^{   \text{KG}  ,n}\\ 
&=&\sum_{i=0}^{n-1} \rho^{   \text{KG}  ,i}+N \rho^{   \text{KG}  , n}, ~n=0,1,...,N-1. \label{ad}
\end{eqnarray}
\end{proposition}

%We use the notation $\psi_Z$ to refer to the partial realization with $\text{dom}(\psi_Z)=Z$.  Define $z^*(Z,\phi)$ as the element $z$ that maximizes $\Delta(z|\psi_Z)$,

We now derive a bound for the adaptive greedy policy by applying linear programming to the problem of minimizing $\frac{F^{\KG}}{F^{\pi^*}}$ subject to the inequalities (\ref{ad}), which is a worst-case analysis. The following lemma states the linear program and its solution. We  use it afterwards to establish the bounds.

\begin{lemma} \label{5}
Given $N \in \mathbb{Z}_+$, consider the following linear program
\begin{eqnarray*} 
&&\min\sum_{i=0}^{N-1}a_i, \\\nonumber
&&\sum_{i=0}^{t-1} a_i+Na_t \geq 1, ~t=0,1,...,N-1.\\\nonumber
\end{eqnarray*}

Then under these $N$ constraints,  $min\sum_{i=0}^{N-1}a_i =1-\alpha^{N}$, where $\alpha=\frac{N-1}{N}$.
\end{lemma}
The proof of this lemma can be found in \cite{nemhauser1978analysis}.

We have the following results, which generalizes the classic result of the greedy algorithm that achieves $(1-1/e)$-approximation to prior-optimality for ranking and selection problems.

\begin{theorem}\label{theo}
Assume we have a budget of $N$ measurements. Let $\pi^*$ denote the optimal sequential policy for the ranking and selection problem, then we have
 $$\frac{F^{KG}}{F^{\pi^*}} \ge 1- (\frac{N-1}{N})^N.$$
\end{theorem}
\begin{proof}
By Proposition \ref{4}, we have $F^{\pi^*} \leq \sum_{i=0}^{n-1} \rho^{\KG, i}+N \rho^{\KG, n}, ~n=0,1,...,N-1.$ Divide by $F^{\pi^*} $ on both sides of this inequality, we have $$1\leq \sum_{i=0}^{n-1} \frac{\rho^{\KG, i}}{F^{\pi^*}}+N\frac{\rho^{\KG, n}}{F^{\pi^*}}, n=0,1,...,N-1.$$  Let $a_i= \frac{\rho^{\KG, i}}{F^{\pi^*}}$, and then these inequalities are identical to the constraints in Lemma \ref{5}. We notice that $$\min \sum_{i=0}^{N-1} a_i = \min \sum_{i=0}^{N-1}\frac{\rho^{\KG, i}}{F^{\pi^*}} \leq \sum_{i=0}^{N-1}\frac{\rho^{\KG, i}}{F^{\pi^*}}=\frac{F^{\KG}}{F^{\pi^*}}.$$ 

By Lemma \ref{5}, we have $\min \sum_{i=0}^{N-1} a_i = 1-\alpha^{N}$, so $\frac{F^{\KG}}{F^{\pi^*}}\geq 1- \alpha^N=1-(\frac{N-1}{N}) ^N$.
\end{proof}

\section{Analysis of Submodularity of the Value of Information}\label{asvi}

The finite-time bounds obtained in the previous sections assume that the value of information is submodular. In general, submodularity does not hold for arbitrary value functions. In this section, we  analyze  the submodularity of the two-alternative case for independent beliefs.  

While submodularity is a property for multi-set functions, we can extend it to any
continuous function by making it possible for the increment to take any positive value. It could be easily extended to any continuous function. This allows us to use results from real analysis to study submodularity.

\begin{definition} A function $f: \mathbb{R}^n \mapsto \mathbb{R}$ is submodular if  for all $x, y \in \mathbb{R}^n$, $x_i \le y_i$ and $\delta \in \mathbb{R}^n_+,$
$$f(x+\delta)-f(x) \ge f(y+\delta) - f(y).$$
\end{definition}

We  show that submodularity of $\mathcal{C}^2$ functions is directly related to its second derivatives and cross-derivatives (the proof is given in Appendix \ref{D}):

\begin{theorem}\label{a17} $\mathcal{C}^2$ function f: $\mathbb{R}^n\rightarrow \mathbb{R}$ is submodular if and only if
every element of its Hessian is non-positive.
\end{theorem}

The concavity of the value of information has been studied extensively by \cite{Frazier:2010:PLM:1898671.1898677}. In this section, we only study the cross-derivatives of the value of information.

Let $M=2$ and the measurement allocation $z=(z_1,z_2)$.
The value of information $v(z)=s(z)f(-\frac{|\theta_1^0-\theta_2^0|}{s(z)})$, where $s(z)=\sqrt{\tilde{\sigma}_1^2(z_1)+\tilde{\sigma}_2^2(z_2)}$,  $\tilde{\sigma}_i^2(z_i)=\frac{\sigma_i^{2,0} z_i}{{\sigma_W^2}/{\sigma_i^{2,0}}+z_i}$, $f(a)=a\Phi(a)+\phi(a)$, $\Phi$ and $\phi$ are the standard normal cumulative distribution and density respectively \cite{Frazier:2010:PLM:1898671.1898677}.

Although the value of information is not concave in general in the two-alternative case, $v$ is concave on the region where all $z_i$'s are large enough (see Theorem 2 in \cite{Frazier:2010:PLM:1898671.1898677}).

We directly calculate the first derivative and cross-derivative of $v$ as
\begin{eqnarray*}
\frac{\partial v}{\partial z_1} &&= \frac{\tilde{\sigma}_1(z_1)\tilde{\sigma}_1'(z_1)}{s(z)}\bigg{[}f(-\frac{|\theta_1^0-\theta_2^0|}{s(z)})+|\theta_1^0-\theta_2^0|\frac{\Phi(-\frac{|\theta_1^0-\theta_2^0|}{s(z)})}{s(z)} \bigg{]}, \\
\frac{\partial^2 v}{\partial z_1 \partial z_2} &&=\frac{\tilde{\sigma}_1(z_1)\tilde{\sigma}_1'(z_1)\tilde{\sigma}_2(z_2)\tilde{\sigma}_2'(z_2)}{s^3(z)} \phi(-\frac{|\theta_1^0-\theta_2^0|}{s(z)})\bigg{(} \frac{|\theta_1^0-\theta_2^0|^2}{\tilde{\sigma}_1^2(z_1)+\tilde{\sigma}_2^2(z_2)}-1\bigg{)}.
\end{eqnarray*}

\begin{theorem}
The value of information is submodular when $M = 2$ and $\theta_{1}^0=\theta_2^0$.
\end{theorem}
\begin{proof}
Concavity of $v(z)$ is proven in Remark 2 by \cite{Frazier:2010:PLM:1898671.1898677}. Since $\theta_{1}^0=\theta_2^0$, $|\theta_{1}^0-\theta_2^0|=0$ and thus $\frac{\partial^2 v}{\partial z_1 \partial z_2} \le 0$. Therefore, $v$ is submodular in this case.
\end{proof}
$\frac{\partial^2 v}{\partial z_1 \partial z_2} \le 0$ is equivalent to $|\theta_1^0-\theta_2^0|^2 \le \tilde{\sigma}_1^2(z_1)+\tilde{\sigma}_2^2(z_2)$. Rewriting this inequality, we get
\begin{equation} \label{fatt}
\frac{1}{\frac{1}{\sigma_1^{2,0}}+\frac{z_1}{\sigma_W^2}}+\frac{1}{\frac{1}{\sigma_2^{2,0}}+\frac{z_2}{\sigma_W^2}} \le \sigma_1^{2,0}+\sigma_2^{2,0}-|\theta_1^0-\theta_2^0|^2.
\end{equation}
We need $ \sigma_1^{2,0}+\sigma_2^{2,0}-|\theta_1^0-\theta_2^0|^2 \ge 0$, which can be achieved by setting our prior variance large enough or using a uniform prior over all alternatives. This is very reasonable when we have very little information about our problem domain.

Inequality equation (\ref{fatt}) defines a region in the $z_1-z_2$ plane. Specifically, this region has the hyperbolic line $\frac{1}{\frac{1}{\sigma_1^{2,0}}+\frac{z_1}{\sigma_W^2}}+\frac{1}{\frac{1}{\sigma_2^{2,0}}+\frac{z_2}{\sigma_W^2}} = \sigma_1^{2,0}+\sigma_2^{2,0}-|\theta_1^0-\theta_2^0|^2$ as its boundary and contains infinity. In particular, when $z_1$ and $z_2$ are large enough (or equivalently when our measurement is accurate enough),  the value of information is submodular.

Since there is no closed-form expression for the value of information under arbitrary allocations,  we cannot verify submodularity in a simple way for problems with more than two alternatives and for correlated beliefs.   Instead, it can be checked using numerical approximation and is easy to guarantee by running repeated experiments and averaging to reduce measurement noise. A necessary condition is the concavity of the value of information for measuring a fixed alternative $x$ for $n$ times, which can be checked exactly.

Intuitively, we may expect that the marginal value of information should decline as we make more observations. But it is not always the case. It is shown  that the value of information for measuring a single alternative  may form an S-curve
which is concave when there are many measurements, but may be convex at the beginning \cite{Frazier:2010:PLM:1898671.1898677}. The
S-curve behavior arises when the measurement noise is large and thus a single measurement simply
contains too little information, leading to algorithmic difficulties and apparent paradoxes. This issue is not related to any specific policy, but rather is an inherent property of learning problems. Although the value of information is not necessarily concave, it can be made concave by measuring
each alternative enough times or (equivalently) using sufficiently precise measurements.

\section{Computational Experiments} \label{sec_numerical} Since the seminal paper by \cite{lai1985asymptotically}, there has been a long history in the optimal learning literature of designing algorithms with provable asymptotic or finite-time bounds \cite{audibert2010best,cappe2013kullback,srinivas2009gaussian,auer2002finite,garivier2008upper,Audibert:2009:ETU:1519541.1519712}.  But  none of these bounds are tight in finite time and different bounds can be based on different metrics. Hence, empirical experiments are needed to better understand the finite time performance of each policy. To this end, we  propose experiments to illustrate the finite time behavior of both KG and other optimal learning policies. We consider the following  learning settings that arise a lot in black box Bayesian optimization.

\textbf{Equal-prior: } $M=100$. The true values $\mu_x$ are uniformly distributed over $[0, 60]$ and measurement noise $\sigma_W=100$. $\theta_x^0=30$ and $\sigma_x^0=10$ for every $x$.

\textbf{Asymmetric unimodular function (AUF):} $x$ is a controllable parameter ranging from 21 to 120.  The objective function is $F(x,\xi)=\theta_1\min(x,\xi)-\theta_2x$,  where $\theta_1$,  $\theta_2$ and the distribution of the random variable $\xi$ are all unknown. The aim is to solve $\max_{x}\mathbb{E}F(x, \xi)$ while learning $\theta_1$, $\theta_2$ and the parameters that determine the distribution of $\xi$. The true distribution of $\xi$ is taken as a normal distribution with mean 60 and standard deviation 18 (corresponding to a 30$\%$ noise ratio).

\textbf{Goldstein-Price's function with additive noise:}  \begin{eqnarray*}
f(x,y,\phi) &=& [1 + (x + y + 1)^2(19 - 14x + 3x^2-14y+6xy+3y^2)]\cdot\\
&&[30 + (2x - 3y)^2
(18-32x + 12x^2+48y-36xy+27y^2)]+\phi,
\end{eqnarray*}
where $-3\le x\le 3$, $-3\le y \le 3$ and  are uniformly discretized into 13 $\times$ 13 alternatives.

In order to obtain the prior distribution,  we follow \cite{jones1998efficient} and \cite{huang2006global} to use Latin hypercube designs for initial fit. For independent beliefs, we adopt a uniform prior with the same mean value $\theta_x^0$ and standard deviation $\sigma_x^0$ for all alternatives. For  correlated beliefs, we use a constant mean value $\theta_x^0$ for all alternatives and  a prior covariance matrix of the form $$\Sigma_{xx'}^0=\sigma e^{-\sum_{i=1}^{d}\lambda_i(x_i-x_i')^2},$$
where each arm $x$ is a $d$-dimensional vector and $\sigma, \lambda_i$ are constant.   We adopt the rule of thumb by \cite{jones1998efficient} for the default number ($10\times p$) of points, where $p$ is the number of parameters to be estimated. In addition, as suggested by \cite{huang2006global}, to estimate the random errors, after the first $10 \times p$ points are evaluated, we add one replicate at each of the locations where the best $p$ responses are found. Maximum likelihood  estimation is then used to estimate the parameters based on the points in the initial design.

 The policies considered in this section is described as follows.

\textbf{EXPL:} A pure exploration strategy that tests each alternative  equally often.

\textbf{EXPT:} A pure exploitation strategy, $X^{\text{EXPT},n}(S^n) = \arg \max_{x} \hat{\mu}^n_x.$

\textbf{Interval Estimation (IE): } \cite{kaelbling1}
\begin{equation*} \label{IE}
X^{\IE,n}(S^n)=\arg \max_x  \theta^n_x + z_{\alpha/2} \sigma_x^n.
\end{equation*}

\textbf{Kriging:} \cite{huang2006global}

Let $x^*=\arg\max_x(\theta^n_x+\sigma^n_x)$, then
\begin{equation*}
X^{\Kriging,n}(S^n)=\arg \max_x(\theta_x^n-\theta_{x^*}^n)\Phi(\frac{\theta_x^n-\theta_{x^*}^n}{\sigma_x^n})+\sigma_x^n\phi(\frac{\theta_x^n-\theta_{x^*}^n}{\sigma_x^n}),
\end{equation*}
where $\phi$ and $\Phi$ are the standard normal density and cumulative distribution functions.

\textbf{UCB-E: } \cite{audibert2010best}
\begin{equation*} \label{UCB-E}
X^{\UCBE,n}(S^n) =\arg \max_x  \hat{\mu}^n_x + \sqrt{\frac{\alpha}{N_x^n}},
\end{equation*}
where $\hat{\mu}_x^n$, $N_x^n$ are the sample mean of $\mu_x$ and number of times $x$ has been measured up to
time $n$.  The quantity $\hat{\mu}_x^0$ is initialized by measuring each alternative once.

\textbf{SR:} \cite{audibert2010best} Let $A_1=\mathcal{X}$, $\overline{\log}(M)=\frac{1}{2}+\sum_{i=2}^M\frac{1}{i}$, $$n_m=\ceil[\Big]{\frac{1}{\overline{\log}(M)}\frac{n-M}{M+1-m}}.$$

For each phase $m=1,...,M-1$:
\begin{enumerate}
\item For each $x \in A_{m}$, select alternative $x$ for $n_m-n_{m-1}$ rounds.
\item Let $A_{m+1}=A_{m}\setminus \arg\min_{x \in A_m} \hat{\mu}_x.$
\end{enumerate}

\subsection{Finite Time Performance of Different Policies}
Although the theoretical analysis in the previous section is to bound the performance of the knowledge gradient policy to the optimal policy (in theory), the optimal sequential policy is impossible to find in practice. To this end, we compare the value of KG to the expected value of the best alternative $\max_x \mu_x$. Define the opportunity cost (OC$^{\pi}$)  of any policy $\pi$  at any time step $n$ as: $$\text{OC}^{\pi}=\max_x \mu_x - \mu_{\tilde{x}^n},$$
where $\tilde{x}^n=\arg \max_{x}\theta^n_x$. We illustrate the finite time behavior of the KG policy under Equal-prior and AUF with independent normal beliefs. We run KG and calculate the opportunity cost ratio = $\frac{\max_x \mu_x - \mu_{\tilde{x}^n}}{\max_x \mu_x}$ in each iteration. We report the mean with $90\%$ confidence interval  averaged over 1000 experiments in Figure \ref{ocr}.

\begin{figure}
    \centering
   \subfigure[Equal-prior]{\includegraphics[width=0.32\textwidth]{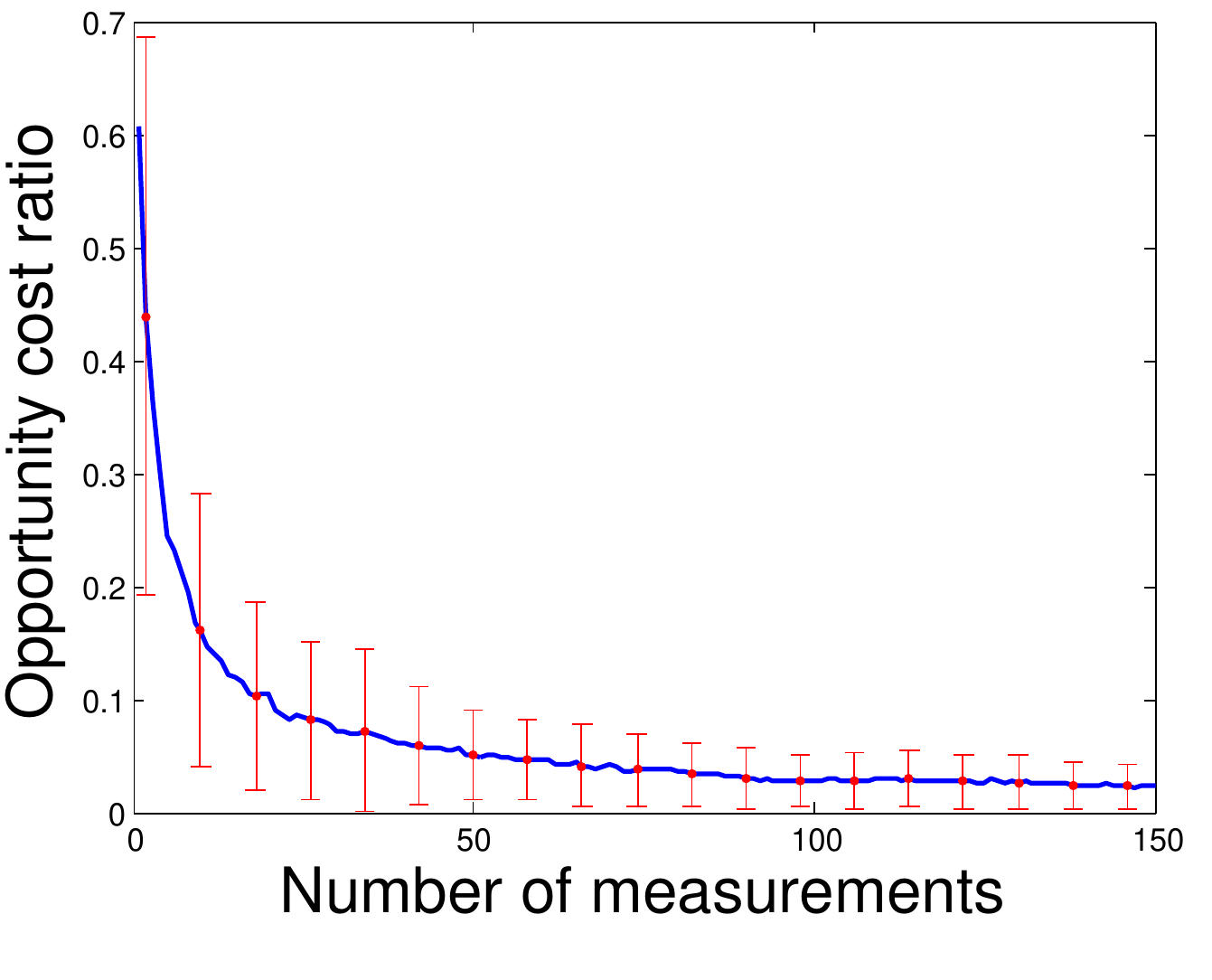}}
   \subfigure[AUF  with $\theta_2=0.2\theta_1$]{   \includegraphics[width=0.32\textwidth]{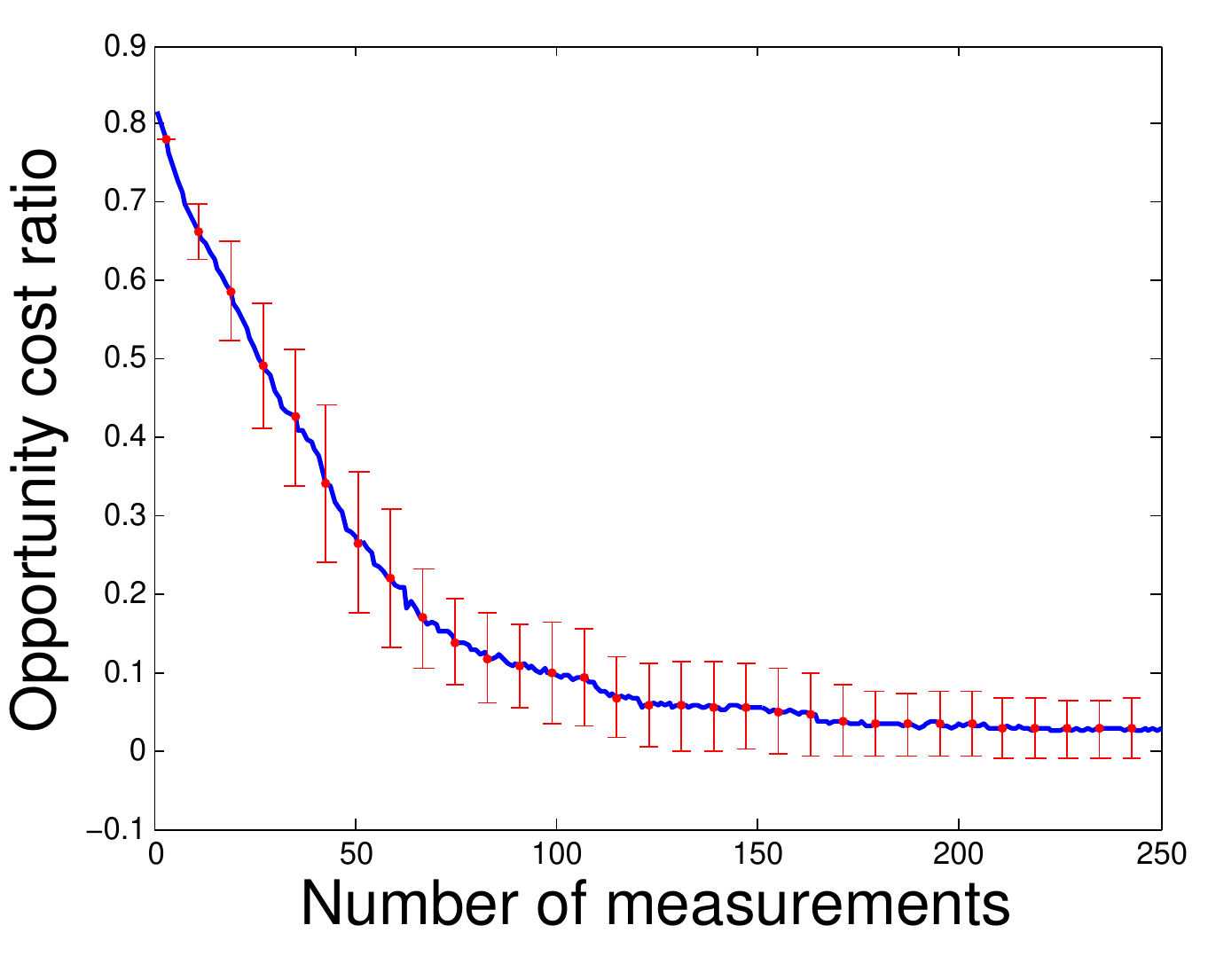}}
      \subfigure[AUF  with $\theta_2=0.8\theta_1$]{   \includegraphics[width=0.32\textwidth]{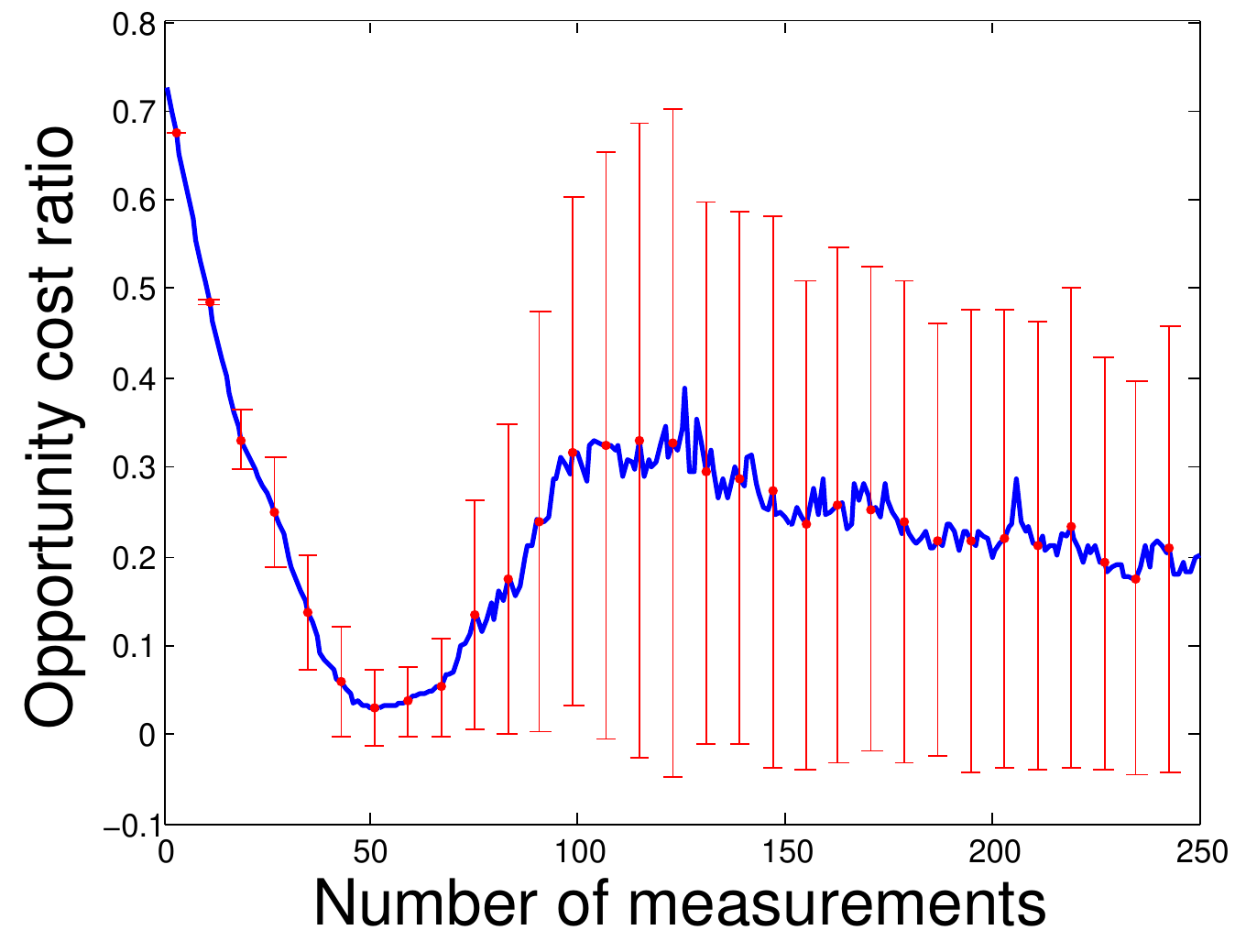}}
    \caption{Opportunity cost ratio. \label{ocr}}
   \label{concavity}
\end{figure}

We next compare the performance of KG, IE with tuning, UCB-E with tuning, SR, EXPL and EXPT. Figure \ref{bu} shows the performance in problem classes AUF and Goldstein  with independent beliefs under a measurement budget  five times the number of alternatives.  
\begin{figure}[hp!]
    \centering
   \subfigure[AUF:  Opportunity cost]{\includegraphics[width=0.4\textwidth]{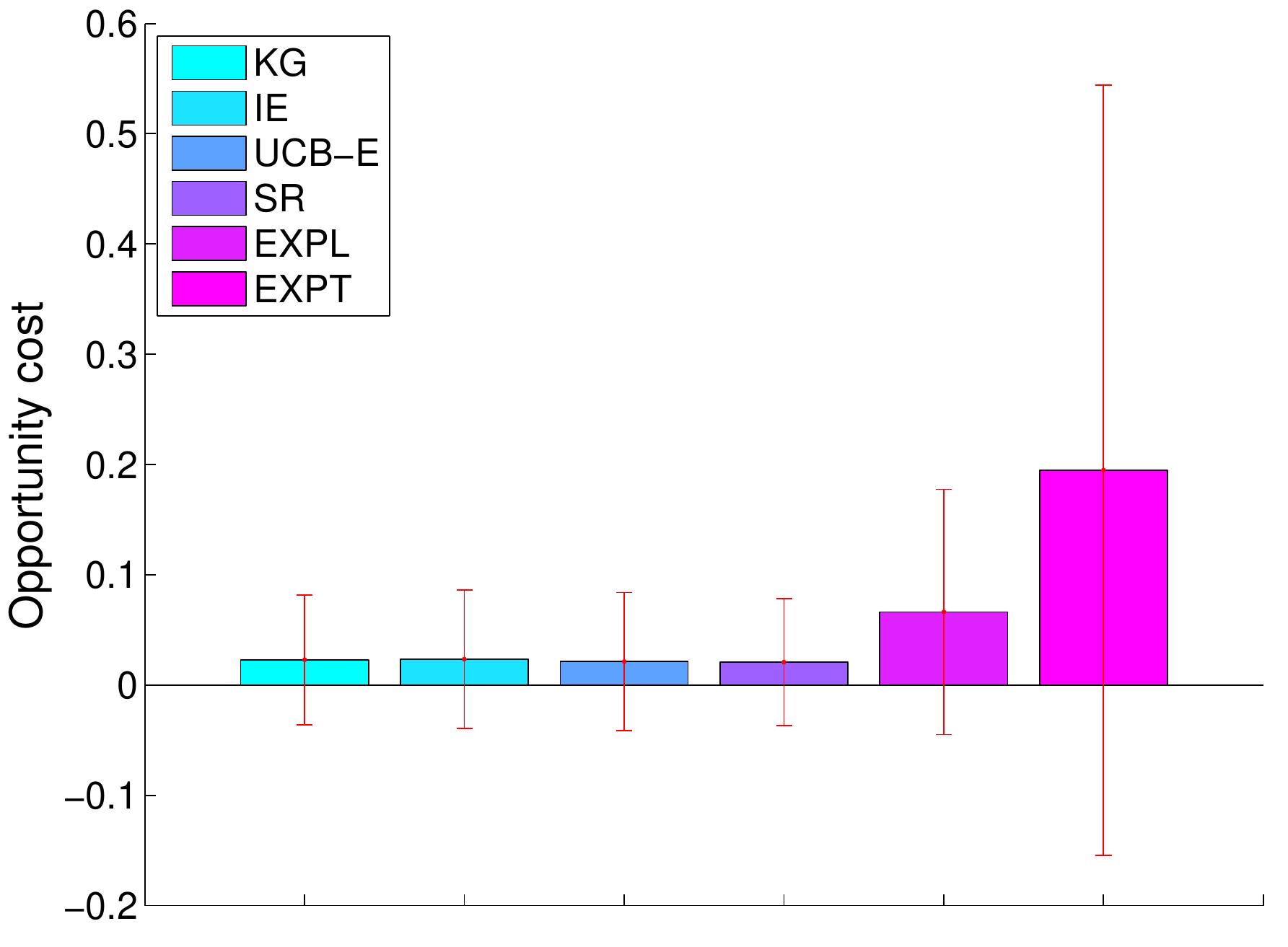}} ~~~
   \subfigure[AUF: Probability of optimality/winning]{   \includegraphics[width=0.4\textwidth]{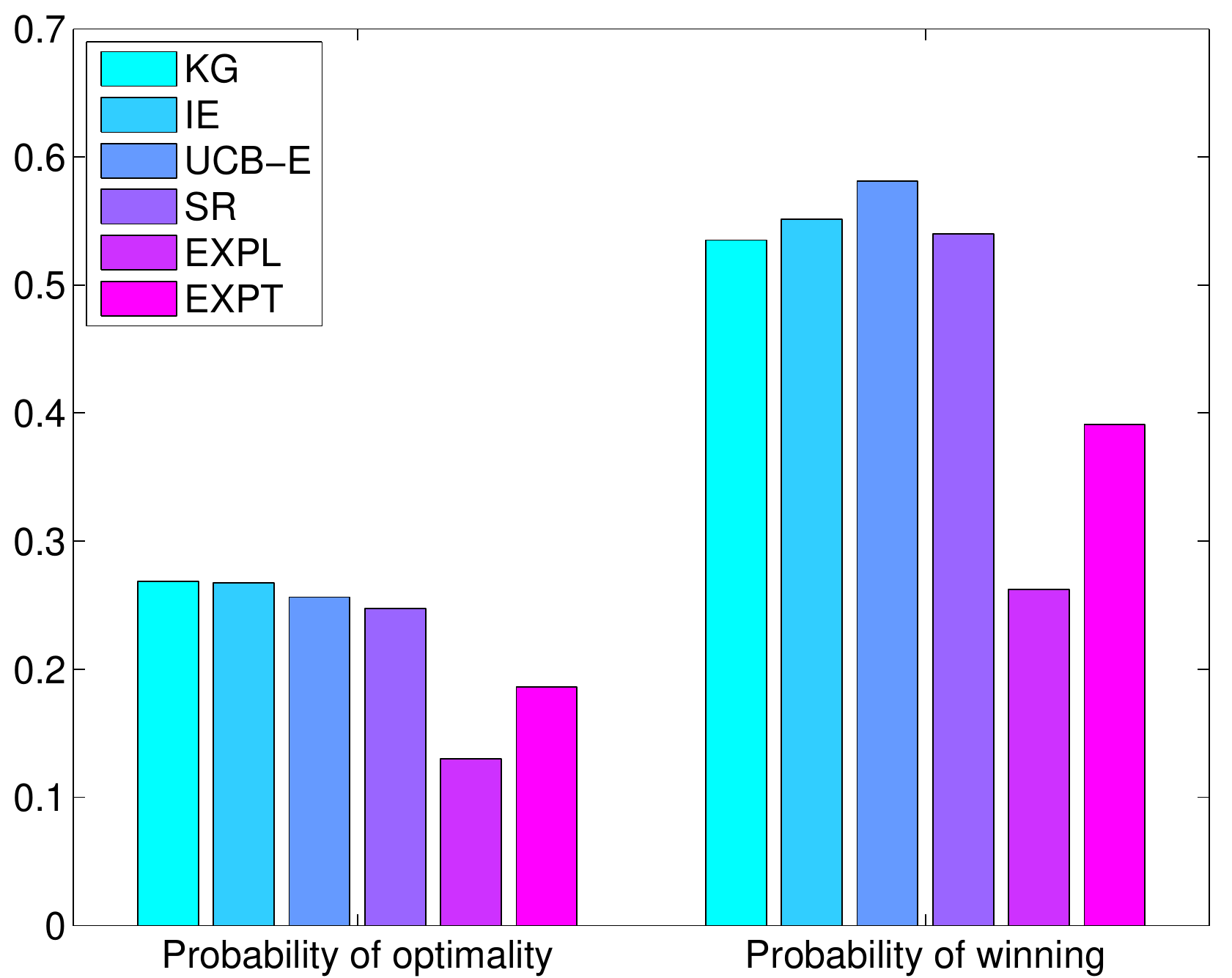}}\\
      \subfigure[Goldstein: Opportunity cost]{   \includegraphics[width=0.4\textwidth]{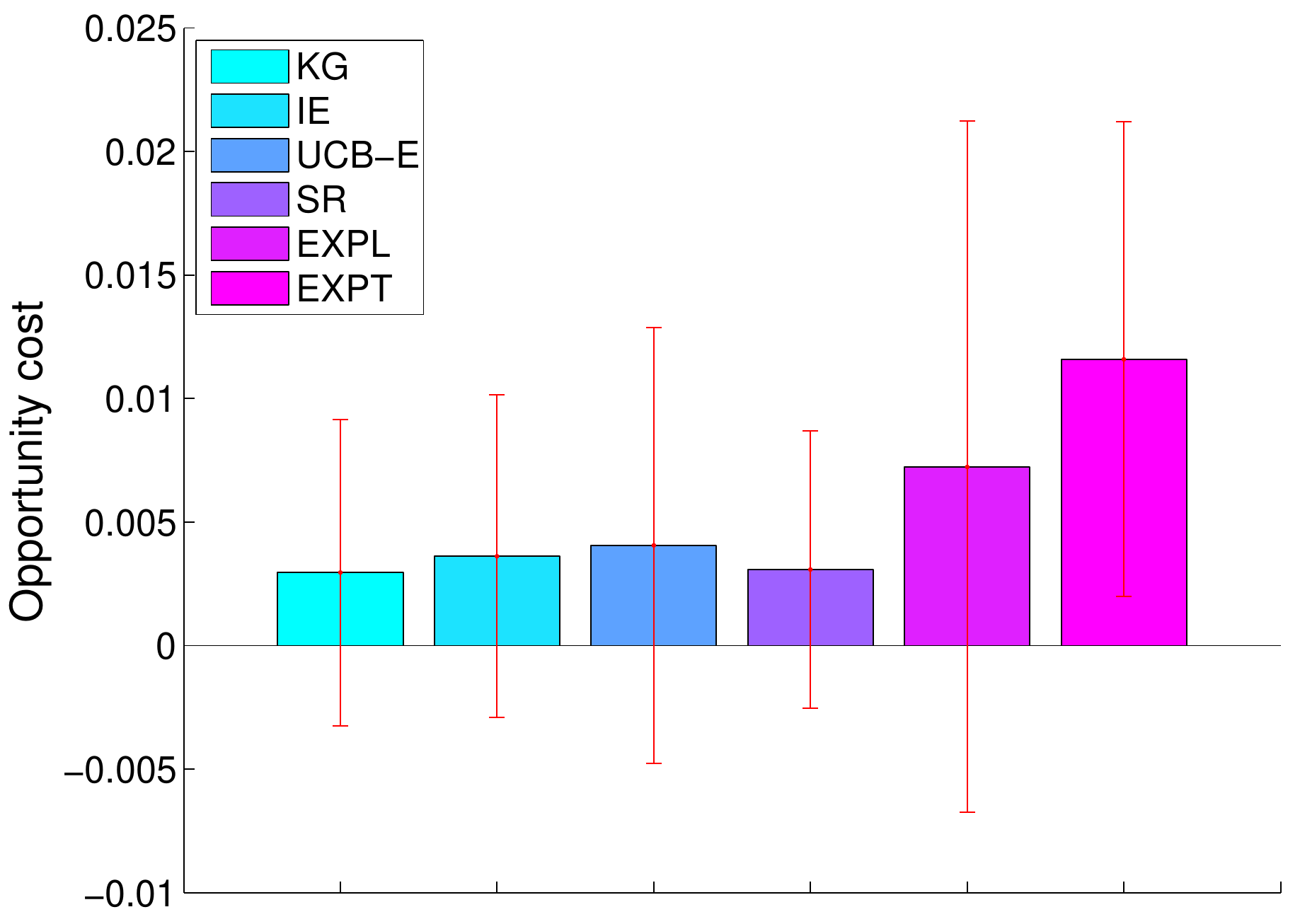}}~~~
   \subfigure[Goldstein:Probability of optimality/winning]{\includegraphics[width=0.4\textwidth]{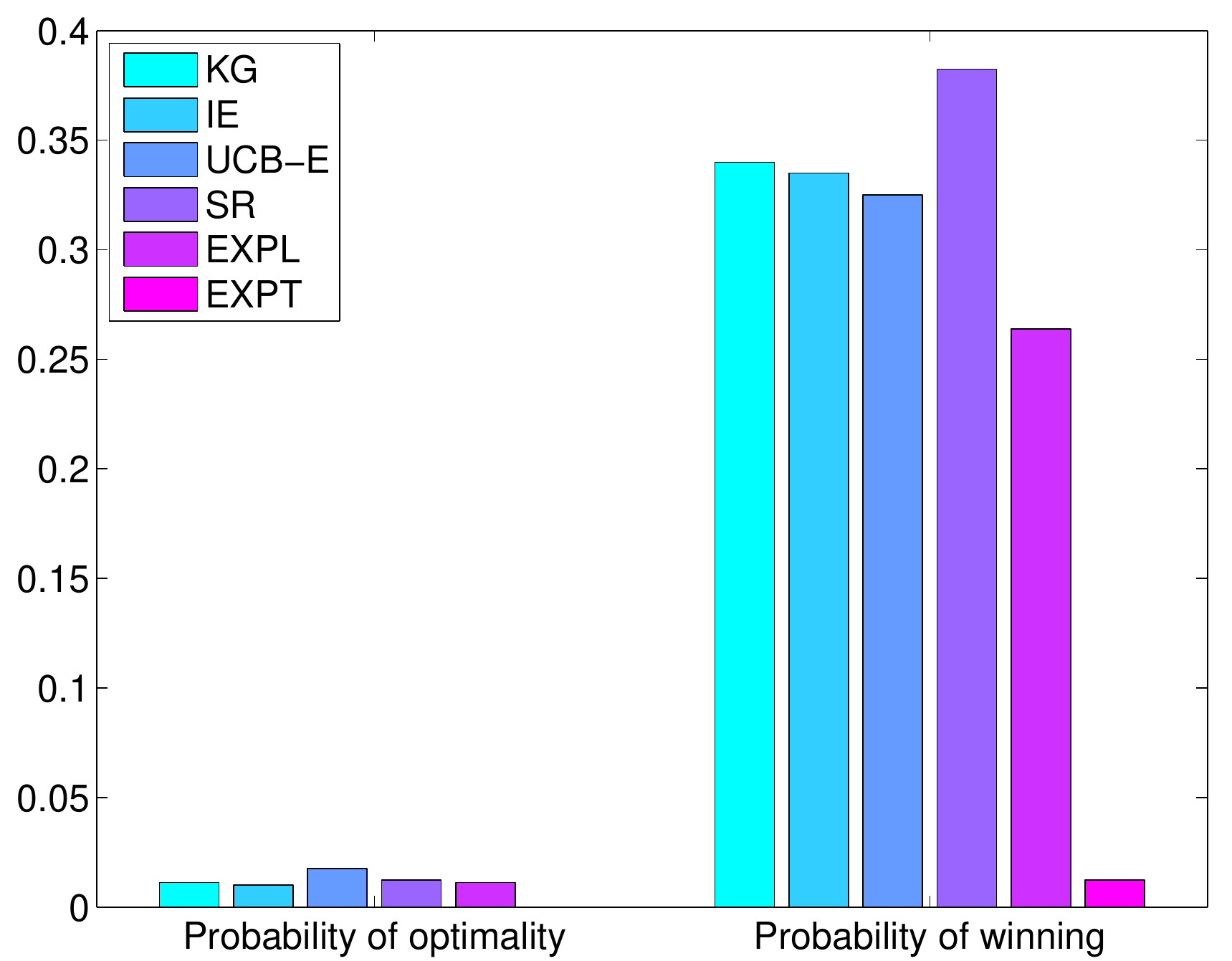}}
    \caption{Comparisons for AUF and Goldstein. (a) and (c) depict the mean opportunity cost with error bars indicating the standard deviation of each policy. The first bar group in (b) and (d) demonstrates the probability that the final recommendation of each policy is the optimal one.  The second bar group in (b) and (d)  illustrates the probability that the opportunity cost of each policy is the lowest. }
    \label{bu}
\end{figure}
We run each policy for 1000 times. In each run, we pre-generate all the observations and share across different policies. We  illustrate in the first column of Figure \ref{bu} the mean opportunity cost and the standard deviation of each policy  over 1000 runs after the measurement budget is exhausted.

In order to give a comprehensive  comparison based on different metrics, we also calculate the probability that the final recommendation of each policy is the optimal one and   the probability that the opportunity cost of each policy is the lowest, as illustrated in the figures on the right hand side of Figure  \ref{bu}. 

The three criteria characterize the behavior of policies from different perspectives. One observation is that there is no universal best policy for all problem classes or under all criteria, which means that theoretical guarantees are not by themselves reliable indicators of which policy is best for a particular problem class.

We also exploit correlated beliefs between alternatives in order to strengthen the effect of each measurement so that one measurement of some alternative can provide information for other alternatives. 

First, we present the OC of different policies  after each iteration under AUF ($\theta_2=0.5\theta_1$) in Figure \ref{correlated_Trend}.    We tune $z_\alpha$ for IE and $\alpha$ for UCB for $N=400$ measurements and the optimal values are $z_\alpha=0.969$ and $\alpha=6.657$. Since UCB-E needs to measure each alternative once, we omit the OC for its first 100 (which is the number of alternatives) steps. KG uses independent beliefs while KGCB, IE and Kriging start from MLE fitted correlated beliefs.  When incorporating correlated beliefs, a measurement of one alternative tells us something about  other alternatives. As a result, KGCB learns faster than KG. 
\begin{figure}[htp!]
    \centering
  \includegraphics[width=0.5\textwidth]{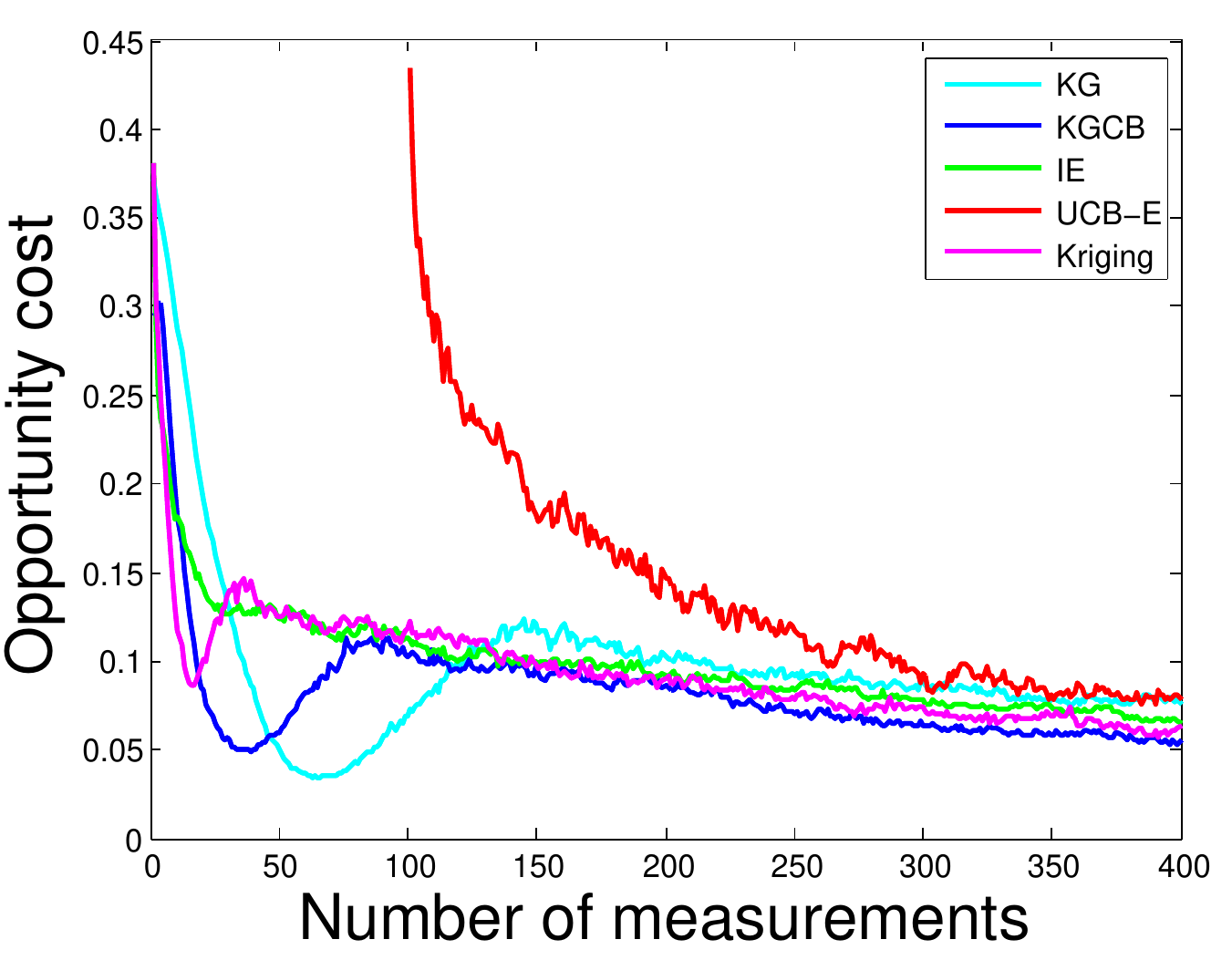}
    \caption{ OC obtained by each policy after each measurement under AUF ($\theta_2=0.5\theta_1$).}
    \label{correlated_Trend}
\end{figure}

\section{Conclusion}
In this paper, we offer a new perspective of interpreting 
 ranking and selection problems as adaptive stochastic multi-set maximization problems. We present the first finite-time bounds for the knowledge gradient on both the posterior optimality and the prior optimality. The prior view provides a cleaner relationship between the performance of the policy and the sample taken, making it possible to relate the value of information to the submodularity of the sample. We analyze the submodularity of the two-alternative case  and provide other conditions for more general problems, bringing out the issue and importance of submodularity in leaning problems.  We propose experiments to further illustrate the finite time behavior of the knowledge gradient policy as well as other policies with or without theoretical guarantees.

\newpage

\appendix
\section{Proofs} 
\subsection{Proof of Lemma \ref{conL}}\label{A}

 For any $\psi$ with $|\psi|=n$, we consider the resulting knowledge state $S^n=(\theta^n_x, \beta^n_x)_{x \in \mathcal{X}}$. Since $\sigma^W \neq 0$, there exists such $\psi$  that $\max_{x}\theta_x^n > \max_{x\neq x'}\theta_x^n$ with positive probability. Now consider another realization $\psi'$ with  $\text{dom}(\psi') =\text{dom}(\psi) \cup \{x_2\}$, where  $x_2$ is the second largest alternative of $\theta_x^n$. We denote the observation of $x_2$ in $\psi'$ as $W_2$ and the resulting $S^{n+1}$ as $(\theta^{n+1}_x, \beta^{n+1}_x)_{x \in \mathcal{X}}$ according to Bayes' rule. The knowledge gradient $\Delta(x|\psi)=\nu_x^{\text{KG},n}$ can be analytically expressed by  $$\nu_x^{\text{KG},n}=\tilde{\sigma}_x^nf(\zeta_x^n),$$
where $\tilde{\sigma}_x^n=\sqrt{(\beta_x^n)^{-1}-(\beta_x^n+\beta^W)^{-1}}$, $\zeta_x^n=-\Big|\frac{\theta^n_x- \max_{x' \neq x} \theta^n_{x'}}{\tilde{\sigma}_x^n}\Big|$ and $f(\zeta)=\zeta\Phi(\zeta)+\phi(\zeta)$.  $\Phi(\zeta)$ and $\phi(\zeta)$ are, respectively, the cumulative standard normal distribution  the standard normal density \cite{frazier2008knowledge}.  We first notice that $f'(\zeta)=\Phi(\zeta) \ge 0$ for any $\zeta \in \mathbb{R}$ so that $f(\zeta)$ is non-decreasing. We next compare $\nu_{x_1}^{\text{KG},n}$ and $\nu_{x_1}^{\text{KG},n+1}$ for  $x_1=\arg\max_x\theta_x^n$. According to Bayes' rule, the precision $\beta$ of $x_2$ changes only when $x_2$ is measured. So we have $\tilde{\sigma}_x^n = \tilde{\sigma}_x^{n+1}.$ Similarly we have all the $\theta_{x}^{n+1}$ unchanged except for alternative $x_2$. By some algebra, it can be shown that for any $W_2$ such that $\theta_{x_2}^{n} < W_2 \le \frac{\beta^{n}_{x_2}}{\beta^W}(\theta_{x_1}^{n}-\theta_{x_2}^{n})+\theta_{x_1}^{n}$, we have $\nu_{x_1}^{\text{KG},n} < \nu_{x_1}^{\text{KG},n+1}.$ Since $\theta_{x_1}^{n}>\theta_{x_2}^{n}$ by construction, such $W_2$ can be obtained with positive probability.

\subsection{Proof of Proposition \ref{a1}}\label{B}
In this appendix, we prove the properties of submodular multi-set functions. 
We prove the equivalence by showing $2)\Rightarrow 1) \Rightarrow 3) \Rightarrow 4) \Rightarrow 2)$.

\begin{itemize}
\item $2)\Rightarrow 1)$. Take $S \subseteq T$ and $T-S=\{x_1,x_2,...,x_r\}$. Then from $3)$ we have
$\rho_x(S) \geq \rho_x(S \cup \{x_1\})$, $\rho_x(S \cup \{x_1\}) \geq \rho_x(S \cup \{x_1,x_2\})$,..., $\rho_x(S \cup \{x_1,x_2,...,x_{r-1}\}) \geq \rho_x(T)$. Summing these $r$ inequalities yields $1)$.

\item $1) \Rightarrow 3)$. For arbitrary $S$ and $T$ with $T-S=\{x_1,x_2,...,x_r\}$ and $S-T=\{y_1,y_2,...,y_q\}$, from $1)$ we have
\begin{eqnarray} \label{8}
\nonumber
g(S \cup T) - g(S) &&= \sum_{t=1}^r [g(S  \cup \{x_1,...,x_t\})-g( S  \cup \{x_1,...,x_{t-1}\})] \\ \nonumber
&&= \sum_{t=1}^r \rho_{x_t}(S\cup \{x_1,...,x_{t-1}\}) \\
&& \leq \sum_{t=1}^r \rho_{x_t}(S) = \sum_{x\in T-S} \rho_x(S).
\end{eqnarray}
And
\begin{eqnarray} \label{9}
\nonumber
g(S \cup T) - g(T) &&= \sum_{t=1}^q [g(T  \cup \{y_1,...,y_t\})-g( T  \cup \{y_1,...,y_{t-1}\})] \\ \nonumber
&&= \sum_{t=1}^q \rho_{y_t}(T\cup \{y_1,...,y_{t}\}-\{y_t\}\}) \\
&& \geq \sum_{t=1}^q \rho_{y_t}(T\cup S - \{y_t\}) = \sum_{x\in S-T} \rho_x(S \cup T - \{x\}).
\end{eqnarray}

Subtracting equation (\ref{9}) from equation (\ref{8}) we get $3)$.

\item $3)\Rightarrow 4)$. If $S \subseteq  T$, $S-T=\emptyset$, and therefore the last term in $3)$ vanishes.

\item $4) \Rightarrow 2)$. Substitute $T=S \cup \{x,y\}$ into $4)$ to obtain $$g(S \cup \{x,y\}) \leq g(S) + \rho_x(S)+\rho_y(S)=\rho_x(S)+g(S\cup \{y\}).$$
Rearrange this inequality, we get
$$\rho_x(S \cup \{y\})= g(S \cup \{x,y\})-g(S\cup \{y\} \leq \rho_x(S). $$
\end{itemize}

\subsection{Proof of Proposition \ref{4}}\label{C}

Let $z^*(Z, \pi, \Phi)$ be the next adaptive greedy choice that maximizes the expected marginal increment given that policy $\pi$ has generated $Z$. We first show that
$$F^{\pi_2@\pi_1}\le F^{\pi_2}+n_1\sum_{Z\in\mathcal{Z}^n}\mathbb{P}(\pi_2  \leadsto Z)\Big{(}\mathbb{E}\big{[}\hat{v}(Z\cup \{z^*(Z, \pi_2, \Phi)\},\Phi)\big{]}-v(Z)\Big{)}$$ for all policies $\pi_1$ with a measurement budget $n_1$ and $\pi_2$ with a budget $n_2$ under any prior and probability distribution that
describes a measurement.

\begin{proof}
Let $\pi^{[j]}$ denote the first $j$ measurement decisions under some policy $\pi$. 
First of all we break $F^{\pi_2@\pi_1}-F^{\pi_2}$  into $n_1$ consecutive differences,
$$F^{\pi_2@\pi_1}-F^{\pi_2}=\sum_{j=1}^{n_1}\Big{(}F^{\pi_2@\pi_1^{[j]}}-F^{\pi_2@\pi_1^{[j-1]}}\Big{)}.$$
Similar to what we did in the last lemma, for each difference we have
\begin{eqnarray*}
&&F^{\pi_2@\pi_1^{[j]}}-F^{\pi_2@\pi_1^{[j-1]}}\\&=&
\sum_{Z_1\in\mathcal{Z}^{n_2+j}}\mathbb{P}(\pi_2@\pi_1^{[j]}  \leadsto Z_1)v(Z_1)-
\sum_{Z_2\in\mathcal{Z}^{n_2+j-1}}\mathbb{P}(\pi_2@\pi_1^{[j-1]}  \leadsto Z_2)v(Z_2)\\
&=&
\sum_{Z_1\in\mathcal{Z}^{n_2+j}}\sum_{Z_2\in\mathcal{Z}^{n_2+j-1},Z_2\cup Z_3=Z_1}\mathbb{P}(\pi_2@\pi_1^{[j-1]}  \leadsto Z_2)\mathbb{P}(\pi_1^{\{j\}}  \leadsto Z_3|\pi_2@\pi_1^{[j-1]}  \leadsto Z_2)v(Z_1)\\
&-&\sum_{Z_2\in\mathcal{Z}^{n_2+j-1}}\sum_{Z_3\in\mathcal{Z}^1}\mathbb{P}(\pi_2@\pi_1^{[j-1]}  \leadsto Z_2)\mathbb{P}(\pi_1^{\{j\}}  \leadsto Z_3|\pi_2@\pi_1^{[j-1]}  \leadsto Z_2)v(Z_2)\\
&=&\sum_{Z_2\in\mathcal{Z}^{n_2+j-1}}\sum_{Z_3\in\mathcal{Z}^1}\mathbb{P}(\pi_2@\pi_1^{[j-1]}  \leadsto Z_2)\mathbb{P}(\pi_1^{\{j\}}  \leadsto Z_3|\pi_2@\pi_1^{[j-1]}  \leadsto Z_2)\big{(}v(Z_2\cup Z_3)-v(Z_2)\big{)}.
\end{eqnarray*}
Now we consider all possible pair $(Z_4,Z_5)$ such that $Z_4\in\mathcal{Z}^{n_2}$, $Z_5\in\mathcal{Z}^{j-1}$ and $Z_4\cup Z_5=Z_2$. Notice that the policy $\pi_2@\pi_1^{[j]}$ employs a fresh start at the time $n_2$, therefore
the events before and after time $n_2$ are independent.
Then we have
\begin{eqnarray*}
&&\sum_{Z_2\in\mathcal{Z}^{n_2+j-1}}\sum_{Z_3\in\mathcal{Z}^1}\mathbb{P}(\pi_2@\pi_1^{[j-1]}  \leadsto Z_2)\mathbb{P}(\pi_1^{\{j\}}  \leadsto Z_3|\pi_2@\pi_1^{[j-1]}  \leadsto Z_2)\big{(}v(Z_2\cup Z_3)-v(Z_2)\big{)}\\
&=&\sum_{Z_2\in\mathcal{Z}^{n_2+j-1}}\sum_{Z_4\cup Z_5=Z_2}\sum_{Z_3\in\mathcal{Z}^1}
\mathbb{P}(\pi_2  \leadsto Z_4)\mathbb{P}(\pi_1^{[j-1]}  \leadsto Z_5)\mathbb{P}(\pi_1^{\{j\}}  \leadsto Z_3|\pi_2@\pi_1^{[j-1]}  \leadsto Z_2)\\&&\times\big{(}v(Z_2\cup Z_3)-v(Z_2)\big{)}.
\end{eqnarray*}
Based on the submodular property of function $v$,
we have
$$v(Z_2\cup Z_3)-v(Z_2)\le v(Z_4\cup Z_3)-v(Z_4).$$
Then from the definition of $z^*$, we have
\iffalse
\begin{eqnarray*}
v(Z_4 \cup Z_3)-v(Z_4) &=& \mathbb{E}[\hat{v}(Z_4 \cup Z_3, \Phi) -\hat{v}(Z_4,\Phi)]\\
&=&\mathbb{E}_{\Phi} \big[\mathbb{E}[\hat{v}(Z_4 \cup Z_3, \Phi)-\hat{v}(Z_4,\Phi)|\Phi \sim \psi_{Z_4}]\big] \\
&\le&\mathbb{E}_{\Phi} \big[\mathbb{E}[\hat{v}(Z_4 \cup \{z^*(Z_4,\Phi)\}, \Phi)-\hat{v}(Z_4,\Phi)|\Phi \sim \psi_{Z_4}]\big] \\
&=&\mathbb{E}_\Phi[\hat{v}(Z_4\cup \{z^*(Z_4, \Phi )\}, \Phi)]-v(Z_4).
\end{eqnarray*}
\fi

\begin{eqnarray*}
v(Z_4 \cup Z_3)-v(Z_4) &=& \mathbb{E}[\hat{v}(Z_4 \cup Z_3, \Phi) -\hat{v}(Z_4,\Phi)]\\
&=&\mathbb{E}_{\Phi} \big[\mathbb{E}[\hat{v}(Z_4 \cup Z_3, \Phi)-\hat{v}(Z_4,\Phi)|Z^{\pi_2}(\Phi)=Z_4]\big] \\
&\le&\mathbb{E}_{\Phi} \big[\mathbb{E}[\hat{v}(Z_4 \cup \{z^*(Z_4,\pi_2,\Phi)\}, \Phi)-\hat{v}(Z_4,\Phi)|Z^{\pi_2}(\Phi)=Z_4]\big] \\
&=&\mathbb{E}_\Phi[\hat{v}(Z_4\cup \{z^*(Z_4, \pi_2, \Phi )\}, \Phi)]-v(Z_4).
\end{eqnarray*}

Combining the last two inequalities,
we have
\begin{eqnarray*}
&&\sum_{Z_2\in\mathcal{Z}^{n_2+j-1}}\sum_{Z_4\cup Z_5=Z_2}\sum_{Z_3\in\mathcal{Z}^1}
\mathbb{P}(\pi_2  \leadsto Z_4)\mathbb{P}(\pi_1^{[j-1]}  \leadsto Z_5)\mathbb{P}(\pi_1^{\{j\}}  \leadsto Z_3|\pi_2@\pi_1^{[j-1]}  \leadsto Z_2)\\&&\times \big{(}v(Z_2\cup Z_3)-v(Z_2)\big{)}\\
&\le&
\sum_{Z_2\in\mathcal{Z}^{n_2+j-1}}\sum_{Z_4\cup Z_5=Z_2}\sum_{Z_3\in\mathcal{Z}^1}
\mathbb{P}(\pi_2  \leadsto Z_4)\mathbb{P}(\pi_1^{[j-1]}  \leadsto Z_5)\mathbb{P}(\pi_1^{\{j\}}  \leadsto Z_3|\pi_2@\pi_1^{[j-1]}  \leadsto Z_2)\\&&
\times\big{(} \mathbb{E}\hat{v}(Z_4\cup \{z^*(Z_4, \pi_2,\Phi)\},\Phi)-v(Z_4)\big{)}\\
&=&
\sum_{Z_2\in\mathcal{Z}^{n_2+j-1}}\sum_{Z_4\cup Z_5=Z_2}
\mathbb{P}(\pi_2  \leadsto Z_4)\mathbb{P}(\pi_1^{[j-1]}  \leadsto Z_5)\big{(} \mathbb{E}\hat{v}(Z_4\cup \{z^*(Z_4, \pi_2, \Phi)\},\Phi)-v(Z_4)\big{)}\\&=&
\sum_{Z_4\in\mathcal{Z}^{n_2}}\sum_{Z_5\in\mathcal{Z}^{j-1}}
\mathbb{P}(\pi_2  \leadsto Z_4)\mathbb{P}(\pi_1^{[j-1]}  \leadsto Z_5)\big{(} \mathbb{E}\hat{v}(Z_4\cup \{z^*(Z_4, \pi_2, \Phi )\}, \Phi )-v(Z_4)\big{)}\\
&=&\sum_{Z_4\in\mathcal{Z}^{n_2}}
\mathbb{P}(\pi_2  \leadsto Z_4)\big{(} \mathbb{E}\hat{v}(Z_4\cup \{z^*(Z_4, \pi_2, \Phi)\},\Phi)-v(Z_4)\big{)},
\end{eqnarray*}
and this ends the proof.
\end{proof}

Set $\pi_1=\pi^*$ and $\pi_2 =\   \text{KG}  ^{[n-1]}$ in Lemma \ref{mon} and the above proposition
then what left to show is that
$$F^{   \text{KG}  ^{[n]}}-F^{   \text{KG}  ^{[n-1]}} \ge \sum_{Z\in\mathcal{Z}^n}\mathbb{P}(\pi_2  \leadsto Z)\Big{(}\mathbb{E}\hat{v}(Z\cup \{z^*(Z,\text{KG}^{[n-1]}, \Phi)\},\Phi)-v(Z)\Big{)}.$$
From the definition, the left hand side of the last equation:
\begin{eqnarray*}
F^{   \text{KG}  ^{[n]}}-F^{   \text{KG}  ^{[n-1]}}&=&\sum_{Z_1\in\mathcal{Z}^{n+1}}\mathbb{P}(   \text{KG}    \leadsto Z_1)v(Z_1)-
\sum_{Z_2\in\mathcal{Z}^{n}}\mathbb{P}(   \text{KG}    \leadsto Z_2)v(Z_2)\\
&=&\sum_{Z_2\in\mathcal{Z}^{n}}\sum_{Z_3\in\mathcal{Z}^{1}}\mathbb{P}(   \text{KG}    \leadsto Z_2)\mathbb{P}(   \text{KG}   \leadsto Z_3|\text{KG}   \leadsto Z_2)v(Z_2\cup Z_3)\\&&-
\sum_{Z_2\in\mathcal{Z}^{n}}\mathbb{P}(   \text{KG}    \leadsto Z_2)v(Z_2).\\
\end{eqnarray*}
Now it is enough to show that
\begin{eqnarray*}
&&\sum_{Z_3\in\mathcal{Z}^{1}}\mathbb{P}(   \text{KG}   \leadsto Z_3|\text{KG}   \leadsto Z_2)v(Z_2\cup Z_3)-v(Z_2) \\
\ge
&&~\mathbb{E}\hat{v}(Z_2\cup\{ z^*(Z_2, \text{KG}^{[n-1]}, \Phi )\},\Phi)-v(Z_2).\end{eqnarray*}
We could group together the partial realizations $\psi$ that lead to the  same single step optimal decision $z^*(Z_2, \text{KG}^{[n-1]}, \Phi)$,
and then the last inequality follows from the adaptive greedy nature  of the KG policy.

\subsection{Proof of Theorem \ref{a17}} \label{D}

First of all, we consider the case when $f$ is a two dimensional function and the four points we pick form a rectangle.
Assume $f(x,y)$ is submodular.
For any given point $(x_0,y_0)$, we have $f(x_0+t+s,y_0)-f(x_0+t,y_0) \le f(x_0+s,y_0) - f(x_0,y_0)$  and $f(x_0+t,y_0)-f(x_0,y_0)\le f(x_0+t,y_0+s)-f(x_0,y_0+s)$ for any $s,t >0$.
From the first inequality we get $f_{xx}(x_0,y_0) \le 0$ directly.
From the second inequality, we have $f_x(x_0,y_0)\le f_x(x_0,y_0+s)$, and finally $f_{x,y}(x_0,y_0)\le 0$.
On the other hand, if we have $f_{xy}\le 0$, $f_{xx} \le 0$, for any $(x,y)$, then due to the fact that
$f(x_0+t,y_0+s)-f(x_0+t,y_0)-\big{(} f(x_0,y_0+s)-f(x_0,y_0) \big{)}=\int_{x_0}^{x_0+t}\int_{y_0}^{y_0+s}f_{xy}(u,v)\mathrm{d}u\mathrm{d}v \le 0$,
$f(x_0+t+s,y_0)-f(x_0+t,y_0) - \big{(} f(x_0+s,y_0) - f(x_0,y_0) \big{)}= stf_{xx}(x_0+\xi,y_0) \le 0$,
we obtain the submodularity.

We next consider the general case  when $f$ is $n$ dimensional and the four points only form a parallelogram.
Since the difference between the two marginal values can be decomposed into summation of several marginal value differences whose reference points form rectangles that parallel to coordinate planes, the result for the general case is straightforward from the two dimensional case.

\bibliographystyle{siamplain}
\bibliography{references}
\end{document}